\documentclass{article}

\PassOptionsToPackage{numbers,square,sort&compress}{natbib}

\usepackage[preprint]{neurips_2026}

\usepackage[utf8]{inputenc} 
\usepackage[T1]{fontenc}    
\usepackage{hyperref}       
\usepackage{url}            
\usepackage{booktabs}       
\usepackage{nicefrac}       
\usepackage{microtype}      
\usepackage[table]{xcolor}         
\usepackage{graphicx}
\usepackage{wrapfig}

\usepackage{float}

\usepackage{makecell}
\usepackage{multirow}

\usepackage{caption}
\usepackage{mathtools}

\newcommand{\ignorewidth}[1]{\makebox[0pt][c]{#1}}

\usepackage{amsmath,amsfonts,bm}

\def\eqref#1{equation~\ref{#1}}

\def\1{\bm{1}}

\def\vtheta{{\bm{\theta}}}

\def\vd{{\bm{d}}}

\def\vu{{\bm{u}}}

\def\vx{{\bm{x}}}

\def\vz{{\bm{z}}}

\def\mA{{\bm{A}}}

\def\mI{{\bm{I}}}

\def\mO{{\bm{O}}}
\def\mP{{\bm{P}}}

\def\mU{{\bm{U}}}
\def\mV{{\bm{V}}}

\def\mX{{\bm{X}}}

\def\mZ{{\bm{Z}}}

\def\mSigma{{\bm{\Sigma}}}

\DeclareMathAlphabet{\mathsfit}{\encodingdefault}{\sfdefault}{m}{sl}
\SetMathAlphabet{\mathsfit}{bold}{\encodingdefault}{\sfdefault}{bx}{n}

\newcommand{\E}{\mathbb{E}}

\newcommand{\R}{\mathbb{R}}

\DeclareMathOperator*{\argmax}{arg\,max}
\DeclareMathOperator*{\argmin}{arg\,min}

\DeclareMathOperator{\Tr}{Tr}

\let\originalleft\left
\let\originalright\right
\renewcommand{\left}{\mathopen{}\mathclose\bgroup\originalleft}
\renewcommand{\right}{\aftergroup\egroup\originalright}

\newcommand{\diff}{\mathop{}\!\mathrm{d}}
\let\originalpartial\partial
\renewcommand{\partial}{\mathop{}\!\mathrm{\originalpartial}}

\newcommand{\T}{^\mathrm{T}}

\usepackage{amsthm}

\newtheoremstyle{custom}
{\parskip} 
{0}
{} 
{} 
{\bfseries} 
{.} 
{.5em} 
{} 
\theoremstyle{custom}
\newtheorem{theorem}{Theorem}

\theoremstyle{definition}

\theoremstyle{remark}

\usepackage{enumitem}

\newlength\itemparskip
\newlist{compactitem}{itemize}{1}
\setlist[compactitem]{nosep, leftmargin=1em, itemsep=0pt, parsep=0pt, before={\parskip=\itemparskip}, after=\vspace{-\itemparskip}, label=\textbullet}

\begin{document}

\title{Asymmetric Flow Models}

%

\author{%
  Hansheng Chen \hspace{0.5em} Jan Ackermann \hspace{0.5em} Minseo Kim \hspace{0.5em} Gordon Wetzstein \hspace{0.5em} Leonidas Guibas \vspace{1ex}\\
  Stanford University \\
  \url{https://hanshengchen.com/asymflow}
}

\maketitle

\begin{figure}[H]
    \makeatletter
    \if@anonymous
    \vspace{-5ex}
    \else
    \vspace{-4ex}
    \fi
    \makeatother
    \centering
    \makeatletter
    \if@anonymous
    \includegraphics[width=0.8\linewidth]{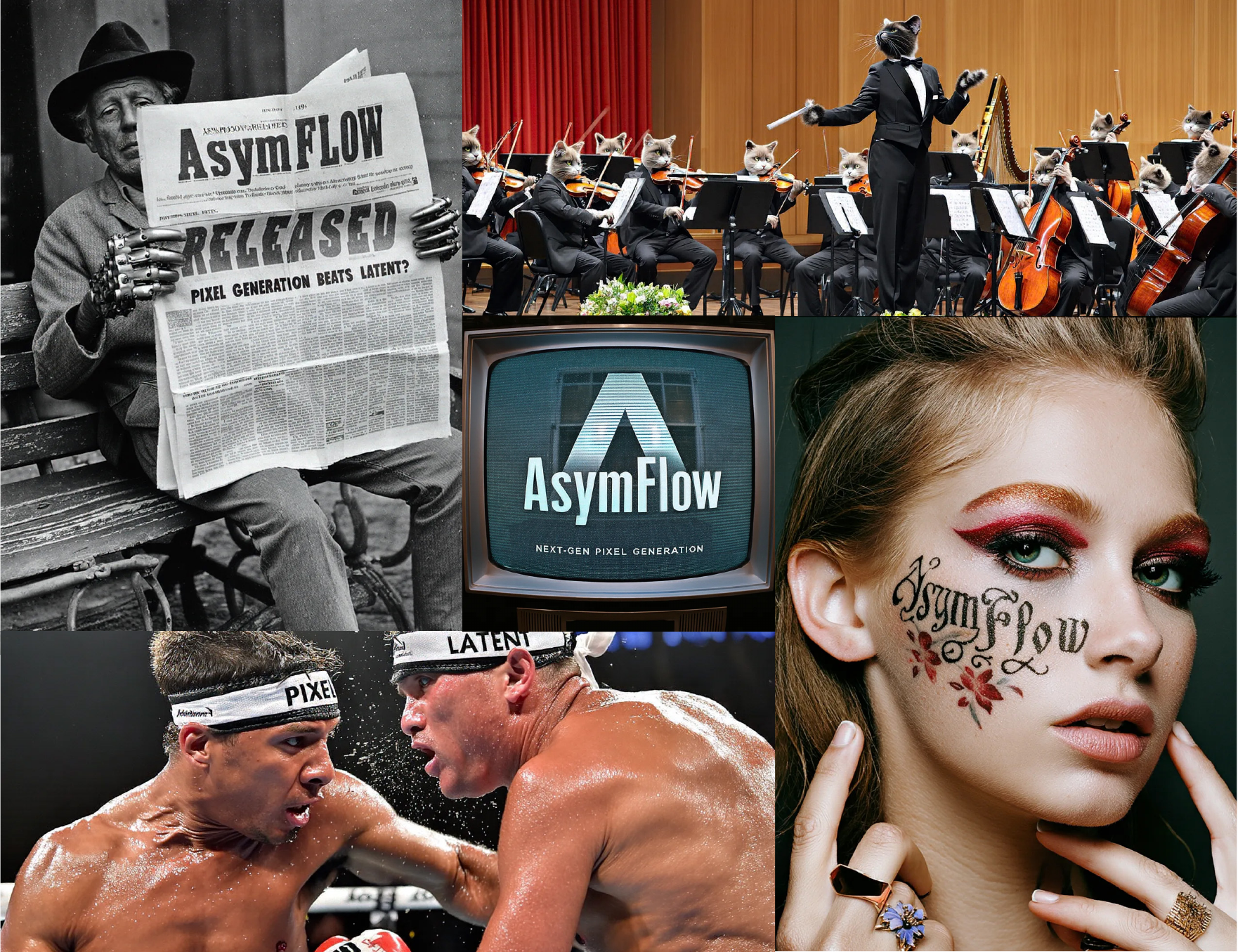}
    \else
    \includegraphics[width=0.82\linewidth]{figures/teaser.pdf}
    \fi
    \makeatother
    \caption{\textbf{AsymFLUX.2 klein generations.} AsymFlow finetunes FLUX.2 klein into a pixel-space flow model, producing highly realistic images with rich visual styles and fine detail.}
    \label{fig:teaser}
\end{figure}

\begin{abstract}
Flow-based generation in high-dimensional pixel spaces is difficult because velocity prediction requires modeling high-dimensional noise, even when data has strong low-rank structure. We present \emph{Asymmetric Flow Modeling} (AsymFlow), a rank-asymmetric velocity parameterization that restricts noise prediction to a low-rank subspace while keeping data prediction full-dimensional. From this asymmetric prediction, AsymFlow analytically recovers the full-dimensional velocity, leaving the network architecture and the standard flow-matching training and sampling procedures unchanged. On ImageNet 256\texttimes 256, AsymFlow achieves a leading 1.57 FID, outperforming prior DiT/JiT-like pixel diffusion models by a large margin. AsymFlow also provides the first-ever route for finetuning pretrained latent flow models into pixel-space models: aligning the low-rank pixel subspace to the latent space gives a seamless initialization that preserves the latent model's high-level semantics and structure, so finetuning mainly improves low-level mismatches rather than relearning pixel generation. We show that the pixel AsymFlow model finetuned from FLUX.2 klein 9B establishes a new state of the art for pixel-space text-to-image generation, beating its latent base on HPSv3, DPG-Bench, and GenEval while qualitatively showing substantially improved visual realism.
\end{abstract}

\section{Introduction}
\label{sec:intro}

Recent progress in diffusion-based image and video generation~\cite{flux2024,wan,hunyuanvideo,ltxvideo,zimage,flux2} has been driven by combining scalable transformer architectures~\cite{dit,videoworldsimulators2024,esser2024sd3} with flow matching objectives~\cite{lipman2023flow,liu2022flow,albergo2023building}. Most state-of-the-art systems operate in compressed lower-dimensional latent spaces learned by autoencoders~\cite{rombach2022ldm}, which is highly scalable but delegates fine detail to a fixed decoder that the generative model cannot control. This limitation motivates a return to high-dimensional generation, including direct pixel-space generation~\cite{li2025jit,chen2025pixelflow,wang2025pixnerd,chen2025dip,yu2025pixeldit,ma2025deco,ma2026pixelgen,baade2026latent,jin2026kdiff}.

However, moving to high-dimensional spaces exposes a bottleneck in velocity prediction. The velocity target $\vu = \bm{\epsilon} - \vx_0$ consists of both data and noise components. To predict it accurately, the network must extract the noise from the input and pass it through its internal features. This is straightforward in latent spaces, where the noise dimension is small relative to the network width. In pixel space, however, the per-patch noise dimension can pollute the network's internal states, creating a bottleneck~\cite{zheng2025rae}. Classical pixel diffusion models used U-Net architectures~\cite{ronneberger2015u,ddpm,dhariwal2021adm,karras2022elucidating,imagen}, whose skip connections naturally route noise from input to output. Modern scalable transformers lack these pathways, so recent methods either reintroduce architectural bypasses, such as U-ViT-like transformers~\cite{bao2023uvit,hoogeboom2023simple,crowson2024scalable,gu2023matryoshka,sid2} or decoder heads~\cite{zheng2025rae, tong2026scaling,wang2025pixnerd,yu2025pixeldit,chen2025dip,ma2025deco}, which complicates the otherwise simple transformer recipe, or switch to predicting clean data $\vx_0$ directly~\cite{li2025jit,ma2026pixelgen,shin2026pixelrepa}, which is numerically ill-conditioned at low noise levels~\cite{karras2022elucidating,salimans2022progressive}.

We introduce \emph{Asymmetric Flow Modeling} (AsymFlow), a new parameterization for high-dimensional flow modeling that avoids both of these compromises. AsymFlow parameterizes the two velocity components asymmetrically: the data component remains full-dimensional, while the noise component is restricted to a low-rank subspace. Concretely, a 16\texttimes 16 RGB patch is a vector in $\R^{768}$, and a rank-$r$ subspace is spanned by $r\ll 768$ orthonormal directions (e.g., $r=8$), leaving only $r$ degrees of freedom for the noise component. The full-dimensional velocity is recovered analytically, so standard flow matching training and sampling remain unchanged. In this view, standard $\vx_0$-prediction and $\vu$-prediction are special cases of AsymFlow, corresponding to zero and full rank of this noise subspace, respectively. Between these endpoints, AsymFlow can choose an intermediate rank that keeps velocity prediction in an important subspace while avoiding full-rank noise prediction.

In addition, AsymFlow makes it possible to build large-scale pixel generators by finetuning pretrained latent flow models. The key observation is that latent and pixel spaces are not disconnected: a latent model can be mathematically lifted into a low-rank pixel model whose samples inherit the semantics and structure of the latent generator. This turns latent-to-pixel adaptation into a correction problem, where finetuning keeps the high-level content and only needs to close the low-level projection gap between low-rank pixel outputs and full-rank pixel targets. To our knowledge, this is the first practical path for turning existing large-scale latent flow models themselves into strong pixel generators.

We evaluate AsymFlow in two settings. On ImageNet 256\texttimes 256~\cite{imagenet}, AsymFlow reaches 1.76 FID with the JiT-H/16 network~\cite{li2025jit} and 1.57 FID with an additional REPA loss~\cite{yu2024repa}, outperforming prior DiT/JiT-like pixel diffusion models by a large margin. For text-to-image generation, our pixel AsymFlow model finetuned from FLUX.2 klein 9B~\cite{flux2} sets a new state of the art in pixel-space generation, beating its latent base on HPSv3~\cite{hpsv3}, DPG-Bench~\cite{ella}, and GenEval~\cite{geneval} while qualitatively exhibiting substantially improved visual realism.

To summarize, our main contributions are:
\begin{compactitem}
    \item We introduce AsymFlow, a novel rank-asymmetric flow parameterization with full-rank data and low-rank noise for scalable high-dimensional generation.
    \item We provide the first method of finetuning pretrained latent flow models into pixel models through AsymFlow, using a principled latent-to-pixel lift without architectural modifications.
    \item We achieve a leading 1.57 FID on ImageNet 256\texttimes 256 and demonstrate a 9B-scale pixel-space text-to-image model with state-of-the-art performance.
\end{compactitem}

\section{Related Work}
\label{sec:related}

Recent work mainly addresses the high-dimensional bottleneck in two ways: changing the network architecture so high-dimensional noisy inputs can reach the output more easily, or changing the prediction parameterization to avoid high-dimensional noise prediction.

\textbf{Hierarchical architectures.}
One line of work keeps noise or velocity prediction feasible using hierarchical architectures with high-dimensional bypasses. Classical DDPM/ADM-style U-Nets~\cite{ddpm,dhariwal2021adm,ronneberger2015u} and U-ViT-like hierarchical transformers~\cite{bao2023uvit,hoogeboom2023simple,crowson2024scalable,gu2023matryoshka,sid2} use skip-connected multi-scale structures, while DDT-like decoder-based designs~\cite{wang2025ddt}, including RAE, PixNerd, PixelDiT, DiP, and DeCo~\cite{zheng2025rae,tong2026scaling,wang2025pixnerd,yu2025pixeldit,chen2025dip,ma2025deco}, expose the noisy input to decoder or refiner pathways conditioned on backbone features. These designs are effective, but they complicate the plain transformer recipe that has scaled successfully in large image and video generators~\cite{flux2024,wan,hunyuanvideo,ltxvideo,zimage,flux2}. In contrast, AsymFlow enables high-dimensional generation without architectural modification, making it possible to finetune large-scale latent flow models into pixel space for the first time.

\textbf{Prediction parameterizations.}
In early diffusion models, hierarchical U-Net-like architectures made $\bm{\epsilon}$-prediction practical, while $\vx_0$-prediction was often less favored because of low-noise numerical issues~\cite{ddpm,salimans2022progressive,karras2022elucidating}. With the paradigm shift to plain diffusion transformers (DiT)~\cite{dit,ma2024sit,yao2025vavae},
JiT~\cite{li2025jit} argues that pixel diffusion should predict clean data $\vx_0$ rather than noise or velocity, and several follow-up pixel methods~\cite{ma2026pixelgen,shin2026pixelrepa} adopt the same $\vx_0$-prediction backbone with perceptual or representation-alignment (REPA) losses~\cite{lpips, yu2024repa}. $k$-Diff~\cite{jin2026kdiff} learns a scalar interpolation between $\vx_0$- and $\vu$-prediction, but this isotropic parameterization does not reduce the dimensionality of the noise component and gives results close to JiT. Unlike prior work, AsymFlow treats the prediction target asymmetrically: the data term $\vx_0$ remains full-dimensional, while the noise term $\bm{\epsilon}$ is restricted to a low-rank subspace, which retains the benefits of $\vu$-prediction in a meaningful subspace.

\section{Preliminaries}

We briefly introduce diffusion models~\citep{sohl2015deep,ddpm,song2019score} using the flow matching convention~\cite{lipman2023flow,liu2022flow,albergo2023building}, then review common prediction parameterizations.

\textbf{Flow matching.} Let $\vx_0 \in \R^D$ be a data vector of dimension $D$. A typical flow model defines an interpolation between a data sample and Gaussian noise $\bm{\epsilon} \sim \mathcal{N}(\bm{0}, \mI)$, yielding the noisy sample $\vx_t \coloneqq \alpha_t \vx_0 + \sigma_t \bm{\epsilon}$, where $t \in (0, 1]$ denotes diffusion time and $\alpha_t = 1 - t$, $\sigma_t = t$ define the linear flow schedule. Under this construction, generative modeling is achieved by solving a reverse-time SDE or ODE that transports noise to data~\citep{song2021scorebased,liu2022rectifiedflowmarginalpreserving}. In particular, the ODE velocity is given by $\frac{\diff \vx_t}{\diff t} = \E_{\vx_0\sim p(\vx_0|\vx_t)}\left[\frac{\vx_t - \vx_0}{t}\right]$, which is the posterior mean of the sample velocity $\vu$:
\begin{equation}
    \vu \coloneqq \frac{\vx_t - \vx_0}{\sigma_t} = \bm{\epsilon} - \vx_0.
    \label{eq:velocity_def}
\end{equation}
Then, a model $(\vx_t,t)\mapsto\hat{\vu}$ is trained to estimate this posterior mean with the flow matching loss:
\begin{equation}
    \mathcal{L}_\mathrm{FM} = \E_{t,\vx_0,\bm{\epsilon}} \left[ \left\| \vu - \hat{\vu} \right\|^2 \right].
    \label{eq:fm_loss}
\end{equation}

\textbf{$\vu$-prediction vs. $\vx_0$-prediction.}
The mapping $(\vx_t,t)\mapsto \hat{\vu}$ is often directly parameterized by a neural network, i.e., $\hat{\vu} \coloneqq G_\vtheta(\vx_t, t)$. This $\vu$-prediction form is widely used in modern latent flow models~\cite{rombach2022ldm,dit,esser2024sd3}, where the representation is compressed. When moved to pixels or other high-dimensional representations, however, the target $\vu=\bm{\epsilon}-\vx_0$ requires predicting a high-dimensional noise component in addition to structured data~\cite{li2025jit,zheng2025rae}. An alternative is $\vx_0$-prediction, where the network predicts clean data $\hat{\vx}_0=G_\vtheta(\vx_t,t)$ and recovers velocity as $\hat{\vu} = (\vx_t-\hat{\vx}_0)/\sigma_t$. This avoids directly regressing Gaussian noise~\cite{li2025jit}, but the $1/\sigma_t$ conversion is ill-conditioned at low noise levels~\cite{karras2022elucidating,salimans2022progressive}, limiting final-sample quality. \citet{shin2026pixelrepa} also claim that REPA-style alignment is less effective in $\vx_0$-prediction pixel models. Thus, $\vu$- and $\vx_0$-prediction expose complementary trade-offs where neither is ideal for high-dimensional generation.

\section{Asymmetric Flow Modeling}
\label{sec:method}

\begin{figure}
    \centering
    \includegraphics[width=1.0\linewidth]{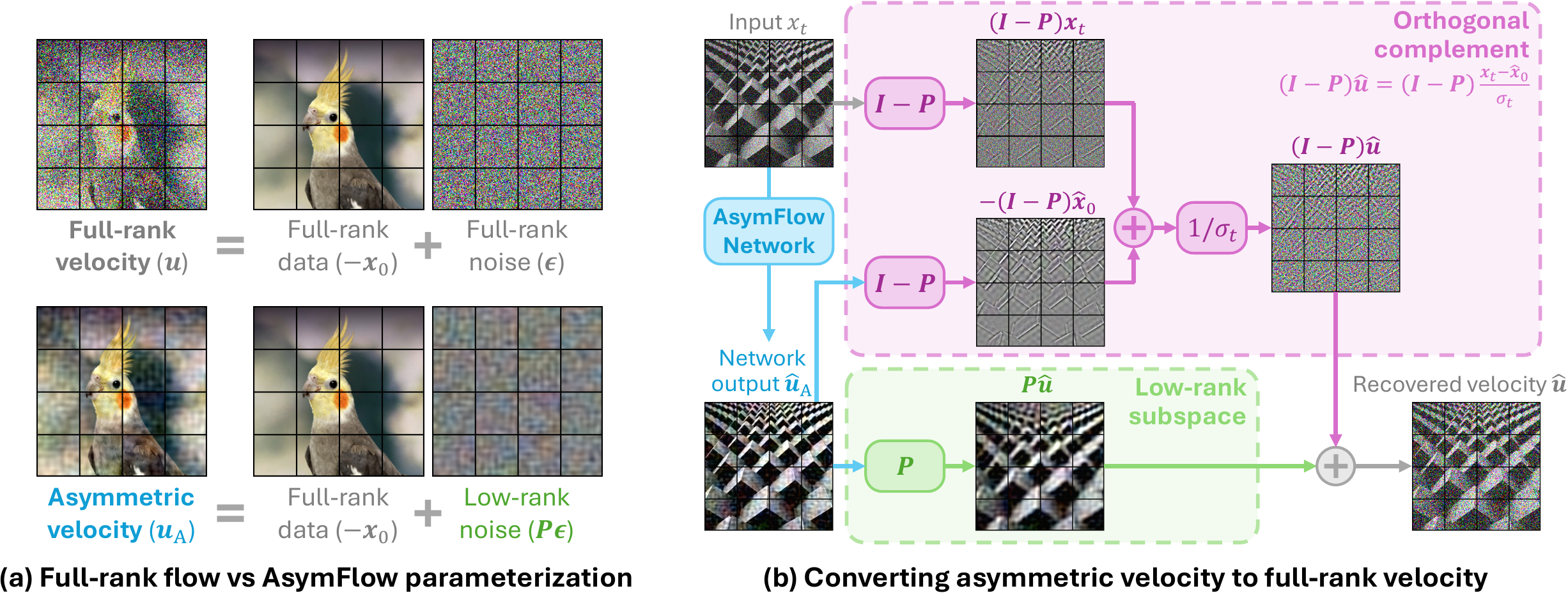}
    \caption{\textbf{AsymFlow parameterization and recovery.}
    (a) AsymFlow changes the standard velocity target by keeping the data term full-dimensional while replacing the noise term with its low-rank projection $\mP\bm{\epsilon}$.
    (b) To recover the full-rank velocity, the low-rank component $\mP\hat{\vu}_\mathrm{A}$ is used directly, while the orthogonal component is converted using the $\vx_0$-to-$\vu$ relation in Eq.~(\ref{eq:velocity_def}).}
    \label{fig:asym_param}
\end{figure}

To address the challenges of high-dimensional flow modeling, we introduce AsymFlow, a rank-asymmetric parameterization of the flow target. The key idea is to treat the two terms in the velocity target asymmetrically: the data prediction term remains full-dimensional, while the noise prediction is restricted to a low-rank subspace. This reduces the burden of representing high-dimensional noise in the network's internal states without changing the network architecture. The full-rank velocity is then recovered analytically for training and sampling, leaving the flow matching formulation unchanged.

\subsection{AsymFlow Parameterization}
\label{sec:asymflow_param}

Let $\mA\in\R^{D\times r}$ be an orthonormal basis of a rank-$r$ subspace, with $\mA\T\mA=\mI_r$, and let $\mP\coloneqq\mA\mA\T$ be the corresponding orthogonal projector. Then $\mathrm{Im}(\mP)$ is the low-rank subspace and
$\mathrm{Im}(\mI-\mP)$ is its orthogonal complement. Given the noise $\bm{\epsilon}\in\R^D$, we use $\mP\bm{\epsilon}$ to denote its subspace component. We refer to $\mP\bm{\epsilon}$ as \emph{low-rank noise}, meaning Gaussian noise projected to a low-rank subspace.

AsymFlow changes the target that the network is asked to predict. In standard $\vu$-prediction (Eq.~(\ref{eq:velocity_def})), the output must reproduce the full noise component $\bm{\epsilon}$ together with the data term $-\vx_0$. For high-dimensional data, this forces the model to carry high-dimensional noise through its features, which pollutes its internal states and wastes network capacity. To address this issue, AsymFlow introduces an \emph{asymmetric velocity} $\vu_\mathrm{A}$ where the noise term is low-rank while the data term remains full-rank:
\begin{equation}
    \vu_\mathrm{A} \coloneqq \mP\bm{\epsilon}-\vx_0 .
    \label{eq:asymflow_target}
\end{equation}
We then train the network to predict the asymmetric velocity, i.e., $\hat{\vu}_\mathrm{A}=G_\vtheta(\vx_t,t)$. 
This prediction will be converted back to the full-rank velocity
$\hat{\vu}$ for loss calculation and denoising sampling (Sec.~\ref{sec:recover_param}).

Fig.~\ref{fig:asym_param}~(a) illustrates the visual difference between the full-rank velocity $\vu$ and the asymmetric velocity $\vu_\mathrm{A}$. Full-rank velocity is perturbed by dense noise, making it highly unpredictable. In contrast, AsymFlow keeps the structured data term full-dimensional but restricts only the stochastic noise term to a low-rank subspace. Since image data itself concentrates near a low-dimensional manifold, this makes the overall asymmetric target more predictable for neural networks.

\textbf{Patch-wise low-rank projection.}
Following the patch-token representation of DiTs~\cite{dit}, we apply low-rank projection independently within each image patch. Concretely, for a patch dimension $D$ and rank $r<D$, the matrix $\mA\in\R^{D\times r}$ defines a low-rank subspace for each patch token, and the same projector $\mP=\mA\mA\T$ is shared across all tokens. Thus, AsymFlow reduces the noise prediction dimension within each patch while preserving the full set of image tokens.

\textbf{Choosing the low-rank subspace.}
When training AsymFlow from scratch, $\mA$ can be obtained from a data-dependent patch basis, e.g., by applying PCA to image patches. When adapting a pretrained latent model, $\mA$ is instead chosen to align the latent space with the pixel patch space, which we compute by a Procrustes alignment between latent variables and their corresponding pixel patches. This latter construction enables a seamless latent-to-pixel initialization, and is discussed in Sec.~\ref{sec:finetune}.

\subsection{Orthogonal Component View and Full-Rank Velocity Recovery}
\label{sec:recover_param}

\begin{figure}
    \centering
    \includegraphics[width=1.0\linewidth]{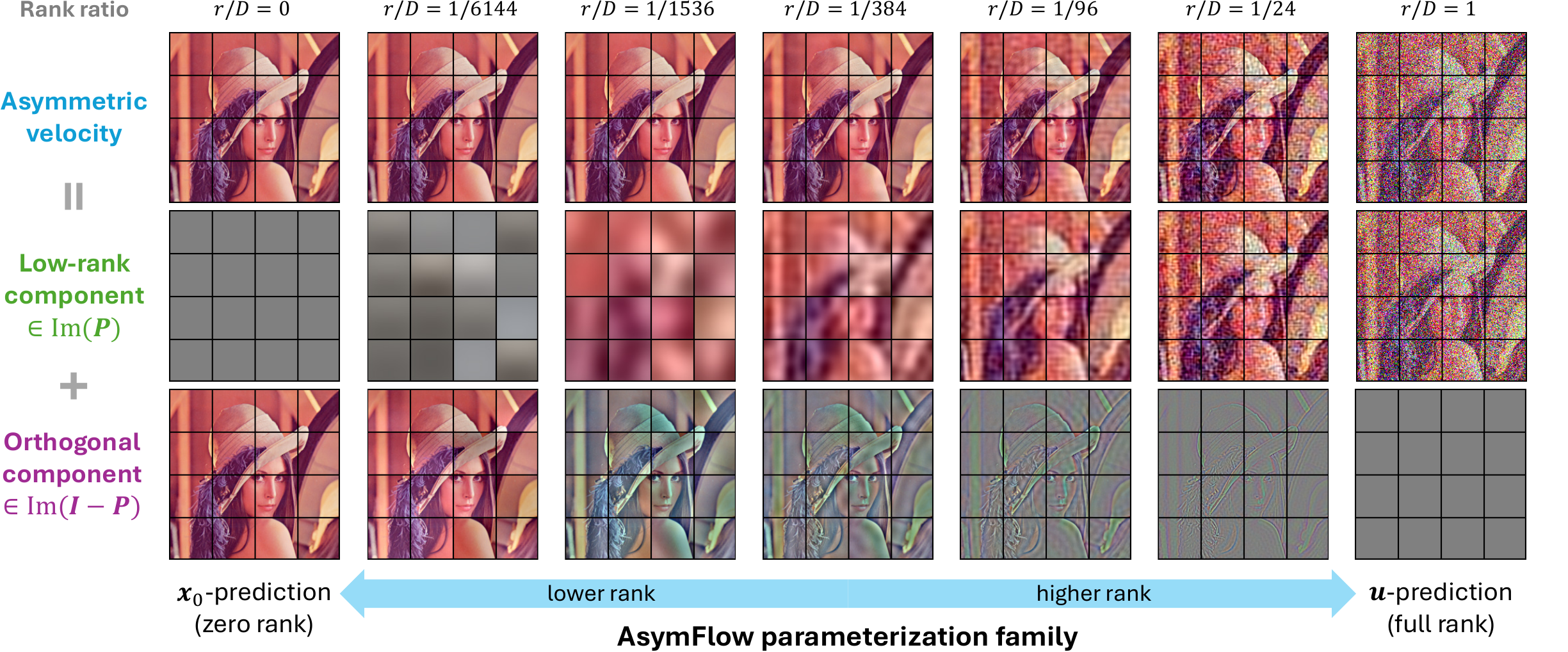}
    \caption{\textbf{Orthogonal component view of AsymFlow}. AsymFlow parameterization can be decomposed into a $\mP\vu$ component in the low-rank subspace $\mathrm{Im}(\mP)$ and an $(\mI-\mP)\vx_0$ component in the orthogonal complement $\mathrm{Im}(\mI-\mP)$. Varying the rank $r$ yields a parameterization family whose endpoints recover full $\vx_0$-prediction and full $\vu$-prediction. 
    }
    \label{fig:param}
\end{figure}

The asymmetric velocity in Eq.~(\ref{eq:asymflow_target}) has a simple interpretation after decomposing it into the low-rank subspace $\mathrm{Im}(\mP)$ and its orthogonal complement $\mathrm{Im}(\mI-\mP)$:
\begin{equation}
    \mP\vu_\mathrm{A}
    = \mP\bm{\epsilon}-\mP\vx_0 = \mP\vu,
    \qquad
    (\mI-\mP)\vu_\mathrm{A}
    =
    -(\mI-\mP)\vx_0 .
    \label{eq:asymflow_component_view}
\end{equation}
The decomposition reveals that AsymFlow behaves like $\vu$-prediction in the low-rank subspace and like $\vx_0$-prediction in the orthogonal complement. Adjusting the rank $r$ creates a family of parameterizations between the two endpoints, as shown in Fig.~\ref{fig:param}: when $r=0$, the target reduces to full $\vx_0$-prediction up to sign; when $r=D$, AsymFlow recovers full $\vu$-prediction. We expect a small but nonzero rank $r$ to be optimal: it retains the benefit of $\vu$-prediction for controlling the flow on a low-dimensional subspace, while avoiding the burden of predicting full-rank noise.

This component view also provides the conversion back to the full-rank velocity. We keep the low-rank velocity component $\mP\vu_\mathrm{A}$, and convert the orthogonal $\vx_0$-style component to velocity using the $\vx_0$-to-$\vu$ relation established in Eq.~(\ref{eq:velocity_def}):
\begin{equation}
    \vu
    =
    \mP\vu_\mathrm{A}
    +
    (\mI-\mP)\frac{\vx_t+\vu_\mathrm{A}}{\sigma_t}.
    \label{eq:asymflow_inference}
\end{equation}

In practice, we apply the conversion to the network prediction $\hat{\vu}_\mathrm{A}$ to obtain $\hat{\vu}$, which is used in the flow matching loss (Eq.~(\ref{eq:fm_loss})) and denoising sampling. Fig.~\ref{fig:asym_param}~(b) illustrates this conversion visually.

\section{Finetuning Latent Flow into Pixel AsymFlow}
\label{sec:finetune}

A key advantage of AsymFlow is that it provides a direct way to turn pretrained $\vu$-predicting latent flow models into pixel-space generators. We first lift a pretrained latent model into an equivalent low-rank pixel flow at initialization, with exact input and output conversions between latents and low-rank pixels. Solving this lifted pixel flow ODE preserves the latent trajectory up to an analytically determined orthogonal noise component, so the initialized model generates lifted low-rank pixels whose semantics and structure match the pretrained latent model. Finetuning then focuses on correcting the low-level projection gap between these low-rank pixels and the full-rank pixel targets.

\subsection{Latent-to-Pixel Initialization}
\label{sec:finetune_init}

We consider a latent flow model $\hat{\vu}_\vz = G_{\bm{\phi}}(\vz_t,t)$ pretrained on latent tokens $\vz_0\in\R^d$ with velocity $\vu_\vz \coloneqq \bm{\epsilon}_\vz-\vz_0$. To bridge the latent-to-pixel gap, we construct a patch-wise linear lift $\mA\in\R^{D\times d}$ from latent space to pixel space using Procrustes alignment (details in Appendix~\ref{app:subspace}), such that the lifted low-rank pixels $\vx_0^\mathrm{L} \coloneqq \mA\vz_0$ approximate the full-rank pixels $\vx_0$. Consider the corresponding pixel-space forward process $\vx_t^\mathrm{L} \coloneqq \alpha_t\vx_0^\mathrm{L}+\sigma_t\bm{\epsilon}$ and velocity $\vu^\mathrm{L}\coloneqq\bm{\epsilon}-\vx_0^\mathrm{L}$. Then the latent and pixel quantities are related by exact input and output conversions:
\begin{equation}
    \text{input:}\quad
    \vz_t=\mA\T\vx_t^\mathrm{L},
    \qquad
    \text{output:}\quad
    \vu^\mathrm{L}
    =
    \mP\mA\vu_\vz
    +
    (\mI-\mP)\frac{\vx_t^\mathrm{L}+\mA\vu_\vz}{\sigma_t}.
    \label{eq:latent_pixel_conversion}
\end{equation}
The input identity shows that noisy low-rank pixels can be projected to noisy latents by $\mA\T$, while the output identity converts the lifted latent velocity $\mA\vu_\vz$ back to the low-rank pixel velocity using the same recovery rule as AsymFlow in Eq.~(\ref{eq:asymflow_inference}). These identities imply trajectory coupling of the lifted pixel and latent ODEs (Theorem~\ref{thm:trajectory_consistency}). Therefore, a $d$-dimensional latent $\vu$-prediction model can be reinterpreted as an exact rank-$d$ pixel flow model with the network $\mA G_{\bm{\phi}}(\mA\T\vx_t^\mathrm{L},t)$. In implementation, the projections $\mA\T$ and $\mA$ are fused into the learnable input and output linear layers of $G_{\bm{\phi}}$, yielding the initialized pixel AsymFlow model $\hat{\vu}_\mathrm{A} = G_{\bm{\theta}}(\vx_t,t)$ for later finetuning.

\begin{wrapfigure}{r}{0.5\textwidth}
    \vspace{-0.3cm}
    \centering
    \includegraphics[width=1.0\linewidth]{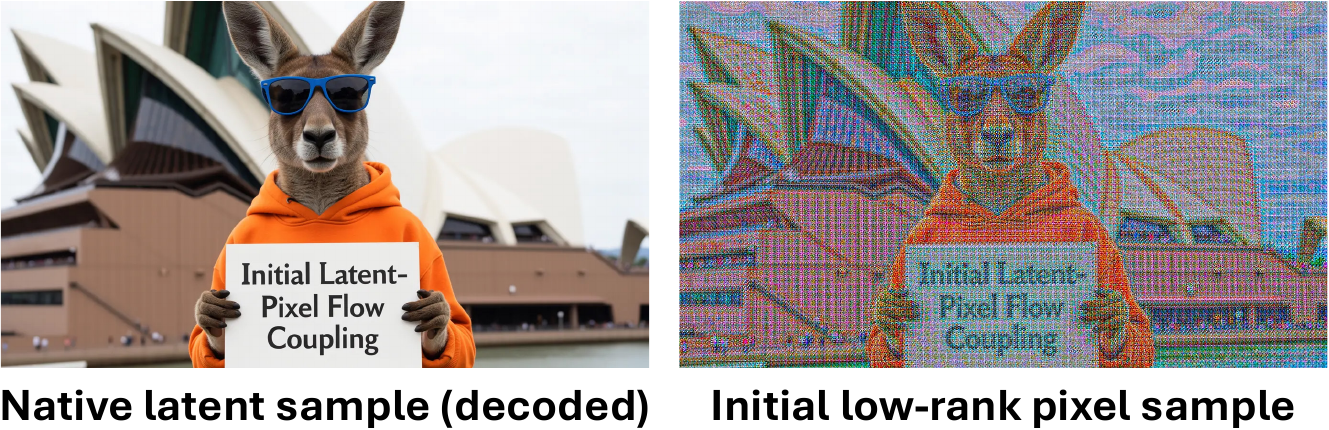}
    \caption{\textbf{Latent-to-pixel initialization.} The lifted low-rank pixel generation are semantically and structurally aligned with the decoded latent generation, leaving only a low-level gap to correct.}
    \label{fig:lowrank_init}
    \vspace{-0.3cm}
\end{wrapfigure}

\textbf{Initialization property.}
The initialized low-rank pixel model predicts a target of the form $\mP\bm{\epsilon}-\vx_0^\mathrm{L}$, so its gap to the AsymFlow target $\vu_\mathrm{A}$ (Eq.~(\ref{eq:asymflow_target})) is only the approximation gap $\vx_0 - \vx_0^\mathrm{L}$. Due to the trajectory coupling (Theorem~\ref{thm:trajectory_consistency}), sampling the initialized model generates $\vx_0^\mathrm{L}$-like lifted low-rank pixel samples without accumulating additional trajectory errors. These samples are semantically and structurally aligned with the $\vx_0$-like decoded latent samples, so the gap $\vx_0-\vx_0^\mathrm{L}$ is mainly low-level and easy to correct during finetuning, as shown in Fig.~\ref{fig:lowrank_init}.

\textbf{Scale calibration.}
A good initialization requires the scale of the lifted pixels $\vx_0^\mathrm{L}$ to align with the scale of real pixels $\vx_0$. However, under the orthonormality constraint $\mA\T\mA=\mI$, Procrustes alignment matches directions but not scale. We therefore introduce a scale factor $s$ and use the scale-calibrated lift $\vx_0^\mathrm{L}=s\mA\vz_0$. In implementation, this scale correction is folded into the model input, output, and internal timestep calibration, as detailed in Appendix~\ref{app:calibration} (ablated in Appendix~\ref{app:scale_ablation}).

\subsection{Variance-Reduced Finetuning Loss}
\label{sec:anchored_fm}

The initialization above reduces latent-to-pixel finetuning to correcting the paired low-level gap $\vx_0-\vx_0^\mathrm{L}$. While the standard flow matching loss (Eq.~(\ref{eq:fm_loss})) regressing to $\vx_0$ already provides a valid objective, the paired low-rank target $\vx_0^\mathrm{L}$ offers additional structure that can be used for variance reduction using control variates, thereby improving convergence and sample quality~\citep{xu2023stable}. 

To achieve this, we inject a term $-\lambda(\vx_0^\mathrm{L}-\E[\vx_0^\mathrm{L}|\vx_t])$ into Eq.~(\ref{eq:fm_loss}). This gives an equivalent flow matching loss whose variance is lower when $\|\vx_0-\vx_0^\mathrm{L}\|$ is small. The conditional mean $\E[\vx_0^\mathrm{L}|\vx_t]$ can then be approximated by the prediction $\hat{\vx}_0^\mathrm{L}$ of a frozen copy of the initialized low-rank model:
\begin{align}
    \E_{t,\vx_0,\bm{\epsilon}}\left[\frac{\left\|\vx_0 - \hat{\vx}_0 - \lambda(\vx_0^\mathrm{L}-\E[\vx_0^\mathrm{L}|\vx_t])\right\|^2}{\sigma_t^2}\right] \approx \E_{t,\vx_0,\bm{\epsilon}}\left[\frac{\left\| \vx_0 - \hat{\vx}_0 - \lambda(\vx_0^\mathrm{L}-\hat{\vx}_0^\mathrm{L}) \right\|^2}{\sigma_t^2}\right] \eqqcolon \mathcal{L}_{\mathrm{VR}}.
    \raisetag{2ex}
    \label{eq:anchored_loss}
\end{align}
Here, $\hat{\vx}_0$ is predicted by the finetuned AsymFlow model from $\vx_t$ (converted to the $\vx_0$ format), and $\hat{\vx}_0^\mathrm{L}$ is predicted by the frozen low-rank model from the paired noisy low-rank sample $\vx_t^\mathrm{L} = \alpha_t \vx_0^\mathrm{L} + \sigma_t \bm{\epsilon}$, diffused with the same noise as $\vx_t$. The parameter $\lambda$ is a patch-wise adaptive weight chosen to minimize the loss gradient norm, thereby reducing the variance of the effective target. In practice, this is implemented via an orthogonal projection and detailed in Appendix~\ref{app:lambda}. Empirically, the resulting variance-reduced objective $\mathcal{L}_{\mathrm{VR}}$ substantially improves fine-grained details in the generated results.

\textbf{Perceptual correction.}
The approximation in Eq.~(\ref{eq:anchored_loss}) assumes $\E[\vx_0^\mathrm{L}|\vx_t] \approx \E[\vx_0^\mathrm{L}|\vx_t^\mathrm{L}]$, which is only exact if $\vx_t-\vx_t^\mathrm{L}\in\mathrm{Im}(\mI-\mP)$. In practice, this condition is rarely strictly satisfied when $t<1$, meaning the variance reduction term $\lambda(\vx_0^\mathrm{L} - \hat{\vx}_0^\mathrm{L})$ introduces a bounded approximation error inside the low-rank subspace $\mathrm{Im}(\mP)$. Empirically, this manifests as excessive noise in the generated results. To compensate, we add an LPIPS perceptual loss~\cite{lpips,ma2026pixelgen} between $\vx_0$ and $\hat{\vx}_0$. This perceptual loss is gated by the same patch-wise weight $\lambda$, and we dynamically fade from the variance reduction term to the LPIPS loss across diffusion time. We defer the exact weighting schedule to Appendix~\ref{app:perceptual}.

\section{Experiments}
\label{sec:exp}

We evaluate AsymFlow in two settings: ImageNet pixel models trained from scratch with the JiT-H/16 network, which isolate the parameterization itself, and large text-to-image models finetuned from the FLUX.2 klein latent generator, which test the finetuning approach and scalability of AsymFlow.

\begin{table}[t]
    \begin{minipage}[t]{0.3\linewidth}
        \vspace{0pt}
        \centering
        \includegraphics[width=0.93\linewidth]{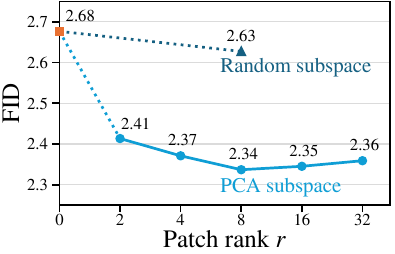}
        \vspace{-0.7ex}
        \captionof{figure}{\textbf{Patch rank and PCA ablation}. 160 epochs.}
        \label{fig:rank_fid}
    \end{minipage}
    \hfill
    \begin{minipage}[t]{0.3\linewidth}
        \vspace{0pt}
        \centering
        \includegraphics[width=0.93\linewidth]{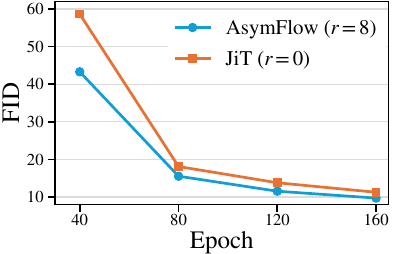}
        \vspace{-0.7ex}
        \captionof{figure}{\textbf{Convergence speed comparison}. Unguided FIDs.}
        \label{fig:converge}
    \end{minipage}
    \hfill
    \begin{minipage}[t]{0.34\linewidth}
        \vspace{-1.5ex}
        \centering
        \caption{\textbf{AsymFlow vs. JiT-H/16} and sensitivity to $\sigma_\mathrm{min}$ clamping. 600 epochs (final checkpoint).}
        \label{tab:sigma_clamp}
        \scalebox{0.8}{%
            \setlength{\tabcolsep}{0.6em}
            \begin{tabular}{lcccccc}
            \toprule
            \textbf{Method} & \textbf{$\sigma_\mathrm{min}$} & \textbf{FID} & \textbf{IS} \\
            \midrule
            \multirow{2}[0]{*}{\makecell[l]{AsymFlow \\ ($r=8$)}} 
            & 0.04 & 1.76 & 312.0 \\
            & \cellcolor{gray!20}0.00 & \cellcolor{gray!20}2.28 & \cellcolor{gray!20}306.2 \\
            \midrule
            \multirow{2}[0]{*}{JiT ($r=0$)} 
            & 0.04 & 1.90 & 300.8 \\
            & \cellcolor{gray!20}0.00 & \cellcolor{gray!20}3.27 & \cellcolor{gray!20}286.7 \\
            \bottomrule
            \end{tabular}
        }
    \end{minipage}
\end{table}

\begin{wraptable}{R}{0.46\linewidth}
    \vspace{-10ex}
    \centering
    \caption{\textbf{ImageNet 256\texttimes 256 pixel diffusion comparison}. FLOP estimation follows the convention in \cite{yu2025pixeldit}. We do not include $k$-Diff~\cite{jin2026kdiff}, as it only reports results with the smaller JiT-B/16 network, where it performs close to the $\vx_0$-prediction JiT baseline. * denotes JiT evaluation protocol, which may have up to 0.08 better FID than ADM according to our tests.}
    \label{tab:imagenet_compare}
    \scalebox{0.75}{%
        \setlength{\tabcolsep}{0.2em}
        \begin{tabular}{lcccl}
        \toprule
        \textbf{Method} & \textbf{Pred (\textpm)} & \textbf{Params} & \textbf{GFLOPs} & \textbf{FID\textdownarrow} \\
        \midrule
        \multicolumn{5}{l}{\textit{Hierarchical CNNs (skip connections / U-Net-like)}} \\
        \midrule
        ADM-G~\cite{dhariwal2021adm} & $\bm{\epsilon}$ & 554M & 2240 & 4.59 \\
        \midrule
        \multicolumn{5}{l}{\textit{Hierarchical transformers (skip connections / U-ViT-like)}} \\
        \midrule
        RIN~\cite{rin} & $\bm{\epsilon}$ & 320M & 668 & 3.42 \\
        SiD, UViT/2~\cite{hoogeboom2023simple} & $\bm{\epsilon}$ & 2B & 1110 & 2.44 \\
        VDM++, UViT/2~\cite{vdm} & $\bm{\epsilon}$ & 2B & 1110 & 2.12 \\
        SiD2, UViT/2~\cite{sid2} & $\bm{\epsilon}$ & - & 274 & 1.73 \\
        EPG-G/16~\cite{epg} & $\vx_0$ & 1.4B & 642 & \underline{1.58} \\
        SiD2, UViT/1~\cite{sid2} & $\bm{\epsilon}$ & - & 1306 & \textbf{1.38} \\
        \midrule
        \multicolumn{5}{l}{\textit{Hierarchical transformers (decoder head / DDT-like)}} \\
        \midrule
        PixNerd-XL/16~\cite{wang2025pixnerd} & $\bm{\epsilon}-\vx_0$ & 700M & 268 & 2.15 \\
        DiP-XL/16~\cite{chen2025dip} & $\bm{\epsilon}-\vx_0$ & 631M & - & 1.79 \\
        DeCo-XL/16~\cite{ma2025deco} & $\bm{\epsilon}-\vx_0$ & 682M & 245 & \underline{1.62} \\
        PixelDiT-XL/16~\cite{yu2025pixeldit} & $\bm{\epsilon}-\vx_0$ & 797M & 311 & \textbf{1.61} \\
        \midrule
        \multicolumn{5}{l}{\textit{Plain transformers (DiT-like)}} \\     
        \midrule
        PixelFlow-XL/4~\cite{chen2025pixelflow} & $\bm{\epsilon}-\vx_0$ & 677M & 5818 & 1.98  \\
        JiT-H/16~\cite{li2025jit} & $\vx_0$ & 953M & 364 & 1.86* \\
        PixelGen-XL/16~\cite{ma2026pixelgen} & $\vx_0$ & 676M & 260 & 1.83 \\
        JiT-G/16~\cite{li2025jit} & $\vx_0$ & 2B & 766 & 1.82* \\
        PixelREPA-H/16~\cite{shin2026pixelrepa} & $\vx_0$ & 953M & 364 & \underline{1.81}* \\
        \textbf{AsymFlow-H/16} & $\mP \bm{\epsilon}-\vx_0$ & 953M & 364 & \textbf{1.57} \\
        \bottomrule
        \end{tabular}
    }
\end{wraptable}

\subsection{Training from Scratch on ImageNet}
We train class-conditional ImageNet 256\texttimes 256 pixel models using the same setup as JiT-H/16 (see Table 9 in \cite{li2025jit}), changing only the prediction parameterization. Unless otherwise stated, AsymFlow is trained using the flow matching loss (Eq.~(\ref{eq:fm_loss})) using a $D=768$ patch-wise PCA subspace of rank $r$, with $r=0$ exactly reproducing JiT's $\vx_0$-prediction. Results use ADM evaluation~\cite{dhariwal2021adm,FID} with grid-searched guidance scales and intervals that optimize FID~\cite{cfg,intervalcfg}. Training for 600 epochs costs approximately 1750 NVIDIA H100 GPU hours, the same as JiT, since the parameterization conversion adds negligible overhead. We defer the details to Appendix~\ref{app:experiments}.

\textbf{Comparison with JiT baseline.}
Table~\ref{tab:sigma_clamp} compares AsymFlow ($r=8$) and the official JiT checkpoint using ADM evaluation after 600 epochs. In practical sampling, the $\vx_0$-to-$\vu$ conversion in Eq.~(\ref{eq:velocity_def}) clamps the denominator by $\sigma_\mathrm{min}$ to avoid numerical instability~\cite{li2025jit}. Since AsymFlow applies this conversion only in the orthogonal complement, it should be less sensitive to this clamp. The results confirm this: with the optimal $\sigma_\mathrm{min}=0.04$ for both methods, AsymFlow improves over JiT in both FID and IS by a clear margin; disabling clamping degrades JiT by 1.37 FID, but AsymFlow by only 0.52. This shows that the asymmetric parameterization improves both overall quality and low-noise numerical stability.

\textbf{Patch rank.}
Figure~\ref{fig:rank_fid} studies the effect of the patch rank. Moving from JiT ($r=0$) to AsymFlow sharply improves guided FID, with the best result at $r=8$; increasing the rank further gives mild degradation. This matches the intended trade-off: AsymFlow keeps velocity prediction in a useful low-rank subspace while avoiding the burden of predicting high-dimensional noise. Notably, every tested rank from $r=2$ to $32$ improves clearly over JiT, so the rank does not require delicate tuning. To check whether the optimal rank transfers across model sizes, we additionally sweep the rank on the much smaller JiT-B/16 network (133M parameters) for 160 epochs: $r=4$, $8$, and $16$ reach FID 4.81, 4.78, and 4.80, respectively. The optimal rank remains $r=8$, with nearby ranks performing closely, suggesting that a coarse sweep suffices when applying AsymFlow to new model sizes.

\textbf{PCA subspace.}
Figure~\ref{fig:rank_fid} also compares PCA and random subspaces at $r=8$. The random subspace performs close to the JiT baseline and far worse than PCA, showing that the gain comes from using a meaningful low-rank subspace, not merely reducing rank.

\textbf{Convergence speed.}
Figure~\ref{fig:converge} compares FID during training. With the same architecture and recipe, AsymFlow ($r=8$) consistently improves over JiT and reaches comparable FID roughly 40\% faster. Thus, the rank-asymmetric target improves not only final quality but also optimization efficiency.

\textbf{Comparison with prior pixel diffusion models.}
Table~\ref{tab:imagenet_compare} compares AsymFlow ($r=8$ plus a standard REPA loss~\cite{yu2024repa}) with prior ImageNet 256\texttimes 256 pixel diffusion models. With REPA, AsymFlow reaches 1.57 FID, establishing the state of the art among practical pixel diffusion models (excluding the much more expensive SiD2 UViT/1). In particular, AsymFlow outperforms previous plain-transformer models by a large margin (FID 1.57 vs. 1.81*). This result also shows that AsymFlow is strongly compatible with REPA: PixelREPA~\cite{shin2026pixelrepa} reports that plain REPA is ineffective for larger JiT models, and its additional designs improve JiT-H/16 only from 1.86* to 1.81* FID; in contrast, adding plain REPA to AsymFlow improves FID from 1.76 to 1.57, suggesting that the AsymFlow parameterization is much more robust to auxiliary losses and can better leverage their benefits.

\subsection{Finetuning Large Text-to-Image Models}

\begin{figure}[t]
    \centering
    \includegraphics[width=1.0\linewidth]{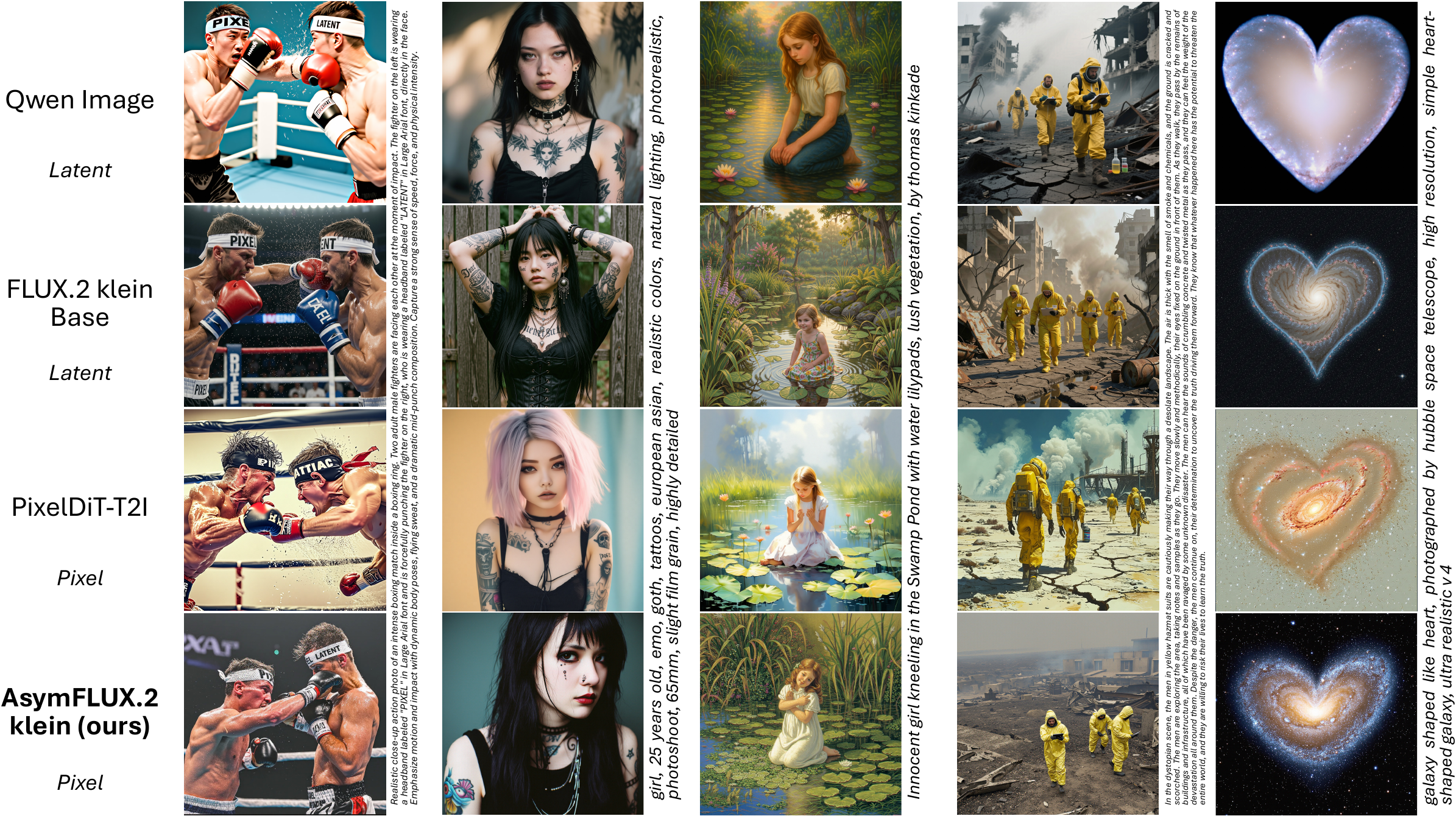}
    \caption{\textbf{Qualitative comparison of T2I diffusion
     models.} AsymFLUX.2 klein produces more realistic images with richer visual styles than prior models. More results are shown in Fig.~\ref{fig:supp_1} and \ref{fig:supp_2}.}
    \label{fig:t2i_system}
\end{figure}

\begin{table}
    \vspace{-3ex}
    \centering
    \caption{\textbf{Comparison with baselines and ablation studies}. All models are finetuned on the LAION-Aesthetics dataset~\cite{laion} for 10K iterations, and evaluated on the COCO-10K dataset~\cite{coco}.}
    \label{tab:flux_ablation}
    \scalebox{0.8}{%
        \setlength{\tabcolsep}{1.2em}
        \begin{tabular}{lcccccc}
        \toprule
        \textbf{Method}
        & \ignorewidth{\textbf{HPSv3\textuparrow}}
        & \ignorewidth{\thickspace\textbf{HPSv2.1\textuparrow}}
        & \ignorewidth{\textbf{VQA\textuparrow}}
        & \ignorewidth{\textbf{CLIP\textuparrow}}
        & \ignorewidth{\textbf{FID\textdownarrow}}
        & \ignorewidth{\textbf{pFID\textdownarrow}} \\
        \midrule
        FLUX.2 klein Base + latent finetune & 10.70 & 0.290 & \textbf{0.936} & 0.276 & \textbf{15.0} & \textbf{18.8} \\
        FLUX.2 klein Base + DDT finetune & 10.33 & 0.291 & 0.922 & 0.273 & 20.4 & 26.0 \\
        \midrule
        AsymFLUX.2 klein (standard FM) & 12.03 & 0.293 & 0.922 & 0.277 & 20.2 & 25.4 \\
        AsymFLUX.2 klein (variance reduction) & \underline{12.99} & \underline{0.296} & \underline{0.925} & \textbf{0.280} & \underline{18.5} & 27.8 \\
        \quad + perceptual correction & \textbf{13.06} & \textbf{0.297} & \underline{0.925} & \underline{0.278} & 19.1 & \underline{22.5} \\
        \bottomrule
        \end{tabular}
    }
\end{table}

\begin{figure}[t]
    \centering
    \includegraphics[width=1.0\linewidth]{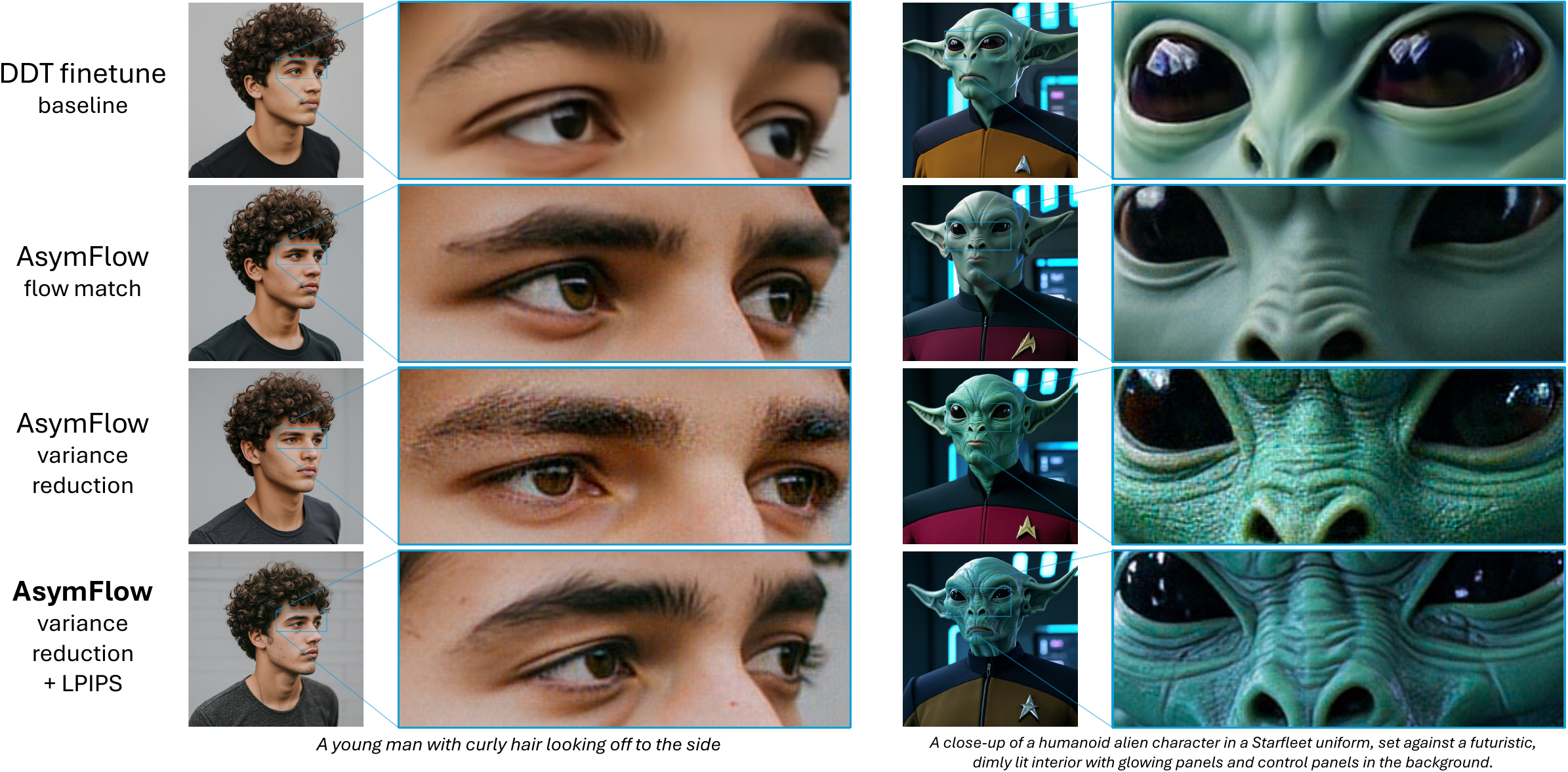}
    \caption{\textbf{Ablation of AsymFLUX.2 klein finetuning}. AsymFlow produces finer details than the DDT baseline. Variance reduction further improves details and texture but introduces excessive noise. The LPIPS perceptual correction suppresses this artifact while preserving the sharp appearance.}
    \label{fig:t2i_ablation}
\end{figure}

For text-to-image generation, we finetune the pretrained FLUX.2 klein Base 9B latent flow model~\cite{flux2} (patch dimension $d=128$) into a pixel-space AsymFlow model. We call the resulting model AsymFLUX.2 klein. The model is finetuned on 3M LAION-Aesthetics images~\cite{laion}, resized to one-megapixel resolution and captioned with Qwen2.5-VL~\cite{qwen25vl}. We freeze the base model and finetune only the input/output projection layers together with rank-256 LoRA adapters~\cite{hu2022lora}. In our preliminary experiments, LoRA performed roughly on par with full finetuning, suggesting that AsymFlow finetuning does not need to deviate far from the pretrained model. The final model is finetuned for 15K iterations at batch size 256, costing approximately 1100 NVIDIA H100 GPU hours. Since AsymFLUX.2 klein keeps the backbone and the number of tokens of FLUX.2 klein, its per-step training and inference costs are the same as the latent model. Sampling uses UniPC~\cite{unipc} with APG orthogonal-projection guidance~\cite{apg}. We defer additional details to Appendix~\ref{app:experiments}.

\begin{wraptable}{R}{0.45\linewidth}
    \vspace{1ex}
    \centering
    \caption{\textbf{System-level comparison of text-to-image (1024\texttimes 1024) diffusion models}.}
    \label{tab:t2i_compare}
    \scalebox{0.8}{%
        \setlength{\tabcolsep}{0.2em}
        \begin{tabular}{lccc}
        \toprule
        \textbf{Method} & \textbf{HPSv3\textuparrow} & \textbf{DPG\textuparrow} & \textbf{GenEval\textuparrow} \\
        \midrule
        \multicolumn{4}{l}{\textit{Latent diffusion models}} \\
        \midrule
        SDXL~\cite{sdxl} & \phantom{0}8.20 & 74.7 & 0.55 \\
        PixArt-$\Sigma$~\cite{pixartsigma} & \phantom{0}\underline{9.37} & 80.5 & 0.54 \\
        Hunyuan-DiT~\cite{hunyuandit} & \phantom{0}8.19 & 78.9 & 0.63 \\
        FLUX.1 dev~\cite{flux2024} & \textbf{10.43} & 84.0 & 0.67 \\
        Qwen-Image~\cite{qwen} & \phantom{0}9.52 & \textbf{87.8} & \textbf{0.86} \\
        FLUX.2 klein Base~\cite{flux2} & \phantom{0}9.50 & \underline{85.2} & \underline{0.80} \\
        \midrule
        \multicolumn{4}{l}{\textit{Pixel diffusion models}} \\
        \midrule
        PixelDiT-T2I~\cite{yu2025pixeldit} & \phantom{0}\underline{8.95} & \underline{83.5} & \underline{0.74} \\
        \textbf{AsymFLUX.2 klein} & \textbf{10.66} & \textbf{86.8} & \textbf{0.82} \\
        \bottomrule
        \end{tabular}
    }
    \vspace{1ex}
\end{wraptable}

\textbf{Evaluation protocol.}
All text-to-image evaluations generate 1024\texttimes 1024 images. For system-level comparison, we use three benchmarks: HPSv3~\cite{hpsv3} measures human preference, which combines realism, style, and overall prompt following, while DPG-Bench~\cite{ella} and GenEval~\cite{geneval} focus more on fine-grained entities, attributes, relations, counting, and composition. For controlled ablations, we generate images using 10K captions from the COCO 2014 validation set~\cite{sdxllightning,coco} and report preference metrics HPSv3~\cite{hpsv3} and HPSv2.1~\cite{hpsv2}, prompt-alignment metrics VQAScore~\cite{vqascore} and CLIP score~\cite{clip}, and distribution metrics FID~\cite{FID} and patch FID (pFID)~\cite{sdxllightning}.

\textbf{System-level comparison.}
Table~\ref{tab:t2i_compare} compares AsymFLUX.2 klein (with variance reduction and perceptual correction) with prior latent and pixel text-to-image diffusion models. AsymFLUX.2 klein improves over its FLUX.2 klein latent base on all three benchmarks, with the largest gain on HPSv3, indicating a substantial improvement in human-aligned visual quality. Consequently, it outperforms the prior pixel model PixelDiT-T2I~\cite{yu2025pixeldit} by a large margin across all metrics, establishing a new state of the art for pixel-space text-to-image generation. Figure~\ref{fig:t2i_system} shows the same trend qualitatively: AsymFLUX.2 klein produces realistic and diverse visual styles with stronger texture, while popular latent models such as Qwen Image~\cite{qwen25vl} and FLUX.2 klein Base~\cite{flux2} still have a more artificial appearance; compared to PixelDiT-T2I, AsymFLUX.2 klein recovers much sharper details in addition to other qualitative improvements, marking a significant step forward for pixel-space text-to-image generation.

\textbf{Controlled baselines.}
To separate dataset effects from latent-to-pixel conversion, we include a latent-finetuned FLUX.2 klein baseline trained on the same data. We also include a $\vu$-prediction pixel finetuning baseline with a DDT decoder head~\cite{wang2025ddt,zheng2025rae}, similar in spirit to PixelDiT~\cite{yu2025pixeldit}. The results are presented in Table~\ref{tab:flux_ablation}: compared to the latent baseline, finetuned AsymFLUX.2 klein models yield clear improvements in HPSv3 and HPSv2.1, indicating that the improved overall quality comes from AsymFlow pixel-space conversion instead of dataset bias. In contrast, the DDT baseline falls behind in all metrics, despite having more parameters and capacity. This is also reflected in the qualitative comparison in Figure~\ref{fig:t2i_ablation}, where the DDT baseline produces blurry images and exhibits minor patch seams, while AsymFLUX.2 klein recovers sharper details and more realistic texture.

\textbf{Loss ablations.}
The results in Table~\ref{tab:flux_ablation} also validate the effectiveness of variance reduction and perceptual correction losses: variance reduction boosts all metrics except pFID, due to its low-noise approximation error that introduces excessive noise (Figure~\ref{fig:t2i_ablation}). This is directly addressed by the LPIPS perceptual correction loss, which significantly improves pFID and HPS scores, resulting in the most natural and realistic texture in Figure~\ref{fig:t2i_ablation}.

\section{Conclusion}
\label{sec:conclusion}

We introduced AsymFlow, a rank-asymmetric flow velocity parameterization that enables high-dimensional pixel-space generation with plain diffusion transformers. When trained from scratch, this single parameterization yields a leading 1.57 FID among ImageNet pixel diffusion models. It also provides the first path for finetuning pretrained large latent flow models into pixel generators with improved visual fidelity, demonstrating AsymFlow's scalability and practical impact. This opens promising directions for high-fidelity image and video generation with finer low-level control, as well as other high-dimensional data modalities previously out of reach for flow-based modeling.

\textbf{Limitations.}
Latent-to-pixel finetuning assumes a good patch-level linear lift. It may not work well when the pretrained latent space does not preserve pixel structure, such as in RAE models~\cite{zheng2025rae}.

\bibliographystyle{plainnat}
\bibliography{main}

\newpage
\appendix

\section{Method Details}
\label{app:details}

\subsection{Low-Rank Subspace Construction}
\label{app:subspace}

For transformer-based pixel generation, AsymFlow requires a patch-wise low-rank subspace. We use two constructions, depending on whether the model is trained from scratch or initialized from a latent model.

\textbf{Orthonormality requirement.}
In both cases we require the columns of $\mA$ to be orthonormal. This ensures that projecting standard pixel-space Gaussian noise preserves its Gaussian form inside the low-rank coordinates: if $\bm{\epsilon}\sim\mathcal{N}(\bm{0},\mI_D)$ and $\mA\T\mA=\mI_r$, then $\mA\T\bm{\epsilon}\sim\mathcal{N}(\bm{0},\mI_r)$.

\textbf{PCA basis for from-scratch training.}
Ideally, the low-rank directions would preserve the most perceptually important information in each image patch. When training from scratch, PCA gives a practical proxy by retaining the dominant patch variations without introducing an additional learned representation. Let $\mX\in\R^{D\times N}$ collect $N$ image patches with normalized pixel values. Taking the top left singular vectors of $\mX$ gives the PCA subspace:
\begin{equation}
    \mX=\mU\mSigma\mV\T,
    \qquad
    \mA=\mU_r,
    \qquad
    \mP=\mA\mA\T.
    \label{eq:app_pca_subspace}
\end{equation}
Here $\mU_r$ denotes the top $r$ columns of $\mU$. Thus $\mP$ keeps the data-adaptive PCA directions and removes the remaining patch-space directions from the noise prediction.

\textbf{Procrustes basis for latent-to-pixel finetuning.}
For latent-to-pixel finetuning, the subspace should be aligned with the pretrained latent representation to minimize the paired gap $\| \vx_0 - \vx_0^\mathrm{L} \|$. Let $\mX\in\R^{D\times N}$ collect image patches with normalized pixel values and $\mZ\in\R^{d\times N}$ collect the corresponding latent tokens. We solve the orthogonal Procrustes problem~\cite{Schönemann_1966}
\begin{equation}
    \mA^\star
    =
    \argmin_{\mA\in\R^{D\times d},\ \mA\T\mA=\mI_d}
    \|\mX-\mA\mZ\|_\mathrm{F}^2 .
    \label{eq:app_procrustes_objective}
\end{equation}
This objective finds an orthonormal lift from latent tokens to pixel patches. Equivalently, it maximizes the inner-product alignment between $\mA\mZ$ and $\mX$, so $\mA^\star = \argmax_{\mA\T\mA=\mI_d}\Tr(\mA\T\mX\mZ\T).$
If $\mX\mZ\T=\mU\mSigma\mV\T$ is the compact SVD, the solution is
\begin{equation}
    \mX\mZ\T=\mU\mSigma\mV\T,
    \qquad
    \mA^\star=\mU\mV\T,
    \qquad
    \mP=\mA^\star(\mA^\star)\T.
    \label{eq:app_procrustes_solution}
\end{equation}
Procrustes aligns directions under the orthonormality constraint. It does not determine the correct pixel scale, so we apply the scalar calibration below.

\subsection{Scale and Timestep Calibration}
\label{app:calibration}

The Procrustes lift gives a directionally aligned low-rank pixel reconstruction, but its magnitude may not match the pixel scale within the Procrustes subspace. We therefore introduce a scalar $s$ and use the calibrated lift
\begin{equation}
    \vx_0^\mathrm{L}=s\mA\vz_0,
    \qquad
    \mA\T\mA=\mI_d,
    \qquad
    \mP=\mA\mA\T.
    \label{eq:app_scaled_lift}
\end{equation}
The scalar $s$ is estimated from the same paired latent-token and pixel-patch statistics used above, by matching the Frobenius norm of the latents $\mZ$ and the rescaled projected pixels $\mA\T\mX / s$:
\begin{equation}
    s
    =
    \frac{\|\mA\T\mX\|_\mathrm{F}}{\|\mZ\|_\mathrm{F}} .
    \label{eq:app_scale_estimator}
\end{equation}
Equivalently, the calibrated lift $s\mA\mZ$ and the low-rank pixels $\mP\mX$ have the same Frobenius norm.

Scale calibration must also be reflected in noisy inputs, not only in the clean lift. Projecting a noisy pixel state gives signal coefficient $s\alpha_t$ and noise coefficient $\sigma_t$, so the latent-space signal-to-noise ratio (SNR) is $s\alpha_t/\sigma_t$. The SNR constraint first determines the latent time $\tau$ at which the pretrained model should be evaluated. Under the linear flow schedule, this gives
\begin{equation}
    \frac{1-\tau}{\tau}
    =
    \frac{s(1-t)}{t}
    \quad\Longrightarrow\quad
    \tau=\frac{t}{s(1-t)+t}.
    \label{eq:app_calibrated_time}
\end{equation}
After fixing $\tau$, the projected input must also have the correct noise magnitude $\sigma_\tau=\tau$. This determines the input rescaling
\begin{equation}
    k
    =
    \frac{\tau}{t}
    =
    \frac{1}{s(1-t)+t},
    \label{eq:app_input_scaling}
\end{equation}
which places the projected state on the latent trajectory expected by the pretrained model, up to a low-rank approximation error:
\begin{equation}
    \mA\T(k\vx_t)
    \approx
    \mA\T(k\vx_t^\mathrm{L})
    =
    \alpha_\tau\vz_0+\sigma_\tau\bm{\epsilon}_\vz
    =
    \vz_\tau.
    \label{eq:app_latent_trajectory_after_calibration}
\end{equation}

The output conversion must use the same calibration. The network is finetuned to predict the calibrated AsymFlow target
\begin{equation}
    \vu_\mathrm{A}^{\mathrm{cal}}
    \coloneqq
    \mP\bm{\epsilon}-\frac{\vx_0}{s},
    \label{eq:app_calibrated_asym_target}
\end{equation}
which is defined in the coordinate system of the rescaled input $k\vx_t$. Recovering the original pixel-space full-rank velocity $\vu=\bm{\epsilon}-\vx_0$ gives
\begin{equation}
    \vu
    =
    \underbrace{
    \mP\left(
    s k\,\vu_\mathrm{A}^{\mathrm{cal}}
    +
    (1-s k)\frac{\vx_t}{\sigma_t}
    \right)
    }_{\text{low-rank subspace}}
    +
    \underbrace{
    (\mI-\mP)\left(
    \frac{\vx_t+s\vu_\mathrm{A}^{\mathrm{cal}}}{\sigma_t}
    \right)
    }_{\text{orthogonal complement}}.
    \label{eq:app_calibrated_velocity_recovery}
\end{equation}
Eq.~(\ref{eq:app_calibrated_velocity_recovery}) is a generalized form of the uncalibrated conversion formula in Eq.~(\ref{eq:asymflow_inference}). When $s=1$ and $k=1$, it reduces to the uncalibrated formula. 

In practice, we apply this generalized conversion to the calibrated network prediction $\hat{\vu}_\mathrm{A}^{\mathrm{cal}} = G_{\bm{\theta}}(k\vx_t,k t)$ to obtain $\hat{\vu}$, which is used in the flow matching loss (Eq.~(\ref{eq:fm_loss})) and denoising sampling.

\subsection{Adaptive Weighting for Variance Reduction}
\label{app:lambda}

The variance-reduced loss in Eq.~(\ref{eq:anchored_loss}) uses a patch-wise coefficient $\lambda$. For a given patch prediction, $\lambda$ is determined by directly minimizing the loss residual along the one-dimensional control-variate direction (see Appendix~\ref{app:vr_deriv} for mathematical justification). Since the gradient of the squared loss is proportional to the corrected residual, this also minimizes the corresponding gradient norm, effectively selecting the lowest-variance target available along that direction.

The one-dimensional minimization has a closed-form solution given by an orthogonal projection. For each patch, define the low-rank prediction deviation of the frozen low-rank model as $\vd^{\mathrm{L}} \coloneqq \vx_0^\mathrm{L}-\hat{\vx}_0^\mathrm{L}$ and the full-rank prediction deviation of the finetuned model as $\vd \coloneqq \vx_0-\mathrm{stopgrad}(\hat{\vx}_0)$. The variance-reduced loss residual is then $\vd - \lambda\vd^{\mathrm{L}}$. Minimizing the patch loss over $\lambda$ gives the one-dimensional least-squares solution:
\begin{equation}
    \lambda^\star
    =
    \argmin_\lambda \|\vd - \lambda\vd^{\mathrm{L}}
    \|^2
    =
    \frac{\langle \vd, \vd^{\mathrm{L}}\rangle}{\|\vd^{\mathrm{L}}\|^2}.
    \label{eq:app_lambda_solution}
\end{equation}
Geometrically, this subtracts the component of the full-pixel prediction deviation that lies along the low-rank prediction deviation, leaving the smallest possible loss residual within this one-dimensional family. In practice, we use the clamped coefficient $\lambda=\min(\max(\lambda^\star,0),1)$.

\subsection{Perceptual Correction}
\label{app:perceptual}

The variance-reduced loss in Eq.~(\ref{eq:anchored_loss}) uses the approximation $\E[\vx_0^\mathrm{L}\mid\vx_t]\approx\E[\vx_0^\mathrm{L}\mid\vx_t^\mathrm{L}]$, as analyzed in Appendix~\ref{app:vr_deriv}. This approximation is valid when $\vx_t-\vx_t^\mathrm{L}\in\mathrm{Im}(\mI-\mP)$, which is guaranteed at $t=1$ because both inputs are pure noise. For $t<1$, this condition requires $\vx_0 - \vx_0^\mathrm{L} \in \mathrm{Im}(\mI - \mP)$, which generally does not hold, so the variance-reduction term $\lambda(\vx_0^\mathrm{L}-\hat{\vx}_0^\mathrm{L})$ can introduce approximation error in the low-rank subspace $\mathrm{Im}(\mP)$. Therefore, we need to reduce reliance on this term near the low-noise end of the trajectory.

Simply downweighting the variance-reduction term near low noise is not ideal, because the variance-reduced target is important for learning fine details. To compensate, we introduce a fading schedule $\omega_t\in[0,1]$ that interpolates from the variance-reduction term to an LPIPS~\cite{lpips} perceptual loss between $\hat{\vx}_0$ and $\vx_0$. The variance-reduction term in Eq.~(\ref{eq:anchored_loss}) is multiplied by $1-\omega_t$:
\begin{equation}
    \mathcal{L}_{\mathrm{VR}}
    =
    \E_{t,\vx_0,\bm{\epsilon}}\left[\frac{\left\| \vx_0 - \hat{\vx}_0 - (1-\omega_t)\lambda(\vx_0^\mathrm{L}-\hat{\vx}_0^\mathrm{L})\right\|^2}{\sigma_t^2}\right],
    \label{eq:app_faded_vr_loss}
\end{equation}
while the complementary perceptual term is multiplied by $\omega_t$:
\begin{equation}
    \mathcal{L}_{\mathrm{P}}
    = \E_{t,\vx_0,\bm{\epsilon}}\left[
    \frac{\omega_t \lambda}{\sigma_t^2}\,
    \mathrm{LPIPS} \left(
    \hat{\vx}_0,\vx_0
    \right)
    \right].
    \label{eq:app_perceptual_loss}
\end{equation}
Here $\lambda$ is reused only as the patch-wise adaptive gate for the perceptual correction, and $1/\sigma_t^2$ recovers velocity-space weighting. 

In our implementation, we define $\omega_t$ as a shifted signal-ratio schedule:
\begin{equation}
    \omega_t
    =
    \frac{\alpha_t^2}{\alpha_t^2 + (\kappa\sigma_t)^2},
    \label{eq:app_shifted_signal_ratio}
\end{equation}
where $\kappa$ is a shift hyperparameter~\cite{esser2024sd3} that controls the transition. The final finetuning loss is
\begin{equation}
    \mathcal{L} = \mathcal{L}_{\mathrm{VR}} + \omega_\mathrm{P} \mathcal{L}_{\mathrm{P}},
\end{equation}
where $\omega_\mathrm{P}$ is a hyperparameter that controls the overall weight of the perceptual correction. In our experiments, we use $\kappa=0.3$ and $\omega_\mathrm{P}=0.2$. We did not perform a systematic hyperparameter sweep due to computational constraints, so there may be room for further improvement.

\section{Experiment Details}
\label{app:experiments}

\subsection{ImageNet Experiments}

For ImageNet 256\texttimes 256 experiments, we use the same architecture, optimizer, and other training hyperparameters as JiT-H/16 (see Table~9 of JiT~\cite{li2025jit}). Training for 600 epochs costs approximately 1750 NVIDIA H100 GPU hours. The REPA-enhanced variant follows the standard REPA setting~\cite{yu2024repa}: we apply the REPA loss to the features after the 8th transformer block with loss weight $0.5$.

At inference time, we set the velocity-recovery clamp to $\sigma_\mathrm{min}=0.04$, which performs better than the JiT default $\sigma_\mathrm{min}=0.05$ for both the JiT baseline and AsymFlow. Unless otherwise stated, all other inference settings follow JiT exactly, including the 50-step Heun ODE solver, class-balanced sampling, BF16 inference, and attention upcasting.

For each classifier-free guidance (CFG)~\cite{cfg} result, we grid-search the CFG scale with step size $0.1$ and the guidance interval with step size $0.02$~\cite{intervalcfg}. Table~\ref{tab:app_imagenet_cfg} lists the selected settings for Fig.~\ref{fig:rank_fid}. The final AsymFlow result in Table~\ref{tab:sigma_clamp} uses CFG scale $2.3$ and interval $[0,0.88]$, while the REPA-enhanced result in Table~\ref{tab:imagenet_compare} uses CFG scale $2.2$ and interval $[0,0.88]$.

\begin{table}[h]
\centering
\caption{\textbf{Guidance settings for the ImageNet patch-rank sweep.} These settings are selected by grid-searching guided FID for each rank.}
\label{tab:app_imagenet_cfg}
\begin{tabular}{lcc}
\toprule
\textbf{Patch rank} $r$ & \textbf{CFG scale} & \textbf{Guidance interval} \\
\midrule
0 & 2.7 & $[0,0.82]$ \\
2 & 2.6 & $[0,0.82]$ \\
4 & 2.6 & $[0,0.82]$ \\
8 & 2.5 & $[0,0.82]$ \\
16 & 2.7 & $[0,0.82]$ \\
32 & 2.7 & $[0,0.82]$ \\
8 (random subspace) & 2.8 & $[0,0.82]$ \\
\bottomrule
\end{tabular}
\end{table}

\subsection{Text-to-Image Experiments}

For text-to-image experiments, we represent pixels in Oklab color space~\cite{oklab} because of its perceptual uniformity, then normalize the values to mean $0$ and standard deviation $1$ before Procrustes alignment and scale calibration. The patch size is $16$, matching the ImageNet model. Thus the pixel patch dimension is $D=16\times16\times3=768$, while the AsymFlow rank follows the original FLUX.2 latent dimension, $r=d=128$.

We finetune on a 3M subset of LAION-Aesthetics images~\cite{laion}, curated with safety and aesthetics filters.
The images are resized to one-megapixel resolution and captioned with Qwen2.5-VL~\cite{qwen25vl}. To reduce overfitting and preserve the pretrained model, we freeze the base weights and update only the input/output projection layers together with rank-256 LoRA adapters~\cite{hu2022lora}. The trained modules are:
\begin{compactitem}
    \item \texttt{x\_embedder}, \texttt{proj\_out}, and \texttt{norm\_out};
    \item rank-256 LoRA adapters with dropout $0.05$ on \texttt{*.ff.linear\_in}, \texttt{*.ff.linear\_out}, \texttt{*.ff\_context.linear\_in}, \texttt{*.ff\_context.linear\_out}, \texttt{timestep\_embedder.linear\_1}, \texttt{timestep\_embedder.linear\_2}, and \texttt{single\_transformer\_blocks.*.attn.to\_out}.
\end{compactitem}

Optimization uses 8-bit Adam~\cite{adam,dettmers20228bit} with batch size $256$, betas $(0.9,0.95)$, learning rate $10^{-4}$ for all trainable parameters (except that \texttt{proj\_out} uses $10^{-3}$). The final model used in the system comparison is trained for 15K iterations, costing approximately 1100 NVIDIA H100 GPU hours. For evaluation, we use the exponential moving average (EMA) of the finetuned weights with the dynamic EMA schedule of~\citet{karras2024analyzingimprovingtrainingdynamics} (using the hyperparameter $\gamma=7.0$). Sampling uses UniPC~\cite{unipc} with APG orthogonal-projection guidance~\cite{apg}. 
At each sampling step, we convert the denoised pixels to RGB color space and clamp the values to the valid range before converting them back to Oklab velocity. Table~\ref{tab:app_t2i_hparams} summarizes the main text-to-image settings.

\begin{table}[h]
\centering
\caption{\textbf{Text-to-image finetuning and evaluation settings.}}
\label{tab:app_t2i_hparams}
\scalebox{0.8}{%
    \begin{tabular}{ll}
    \toprule
    \textbf{Setting} & \textbf{Value} \\
    \midrule
    Pixel color space & Normalized Oklab~\cite{oklab} \\
    Patch size & 16 \\
    Patch dimension $D$ & 768 \\
    Patch rank $r$ & 128 \\
    Subspace construction & Orthogonal Procrustes lift with scale calibration \\
    LoRA rank / dropout & 256 / 0.05 \\
    Flow shift~\cite{esser2024sd3} & 17.0 \\
    \midrule
    Training resolution & 1MP with mixed aspect ratios \\ 
    Pre-shift time sampling & $\mathrm{LogitNormal}(0, 1)$ \\ 
    Optimizer & 8-bit Adam~\cite{adam,dettmers20228bit} \\
    Learning rate & $10^{-4}$ ($10^{-3}$ for \texttt{proj\_out}) \\
    Adam betas & (0.9, 0.95) \\
    Weight decay & 0.0 \\
    Batch size & 256 \\
    Training iterations & 15K iterations \\
    EMA & Dynamic EMA, $\gamma=7.0$~\cite{karras2024analyzingimprovingtrainingdynamics} \\
    \midrule
    Sampler & UniPC~\cite{unipc} \\
    Guidance scale & 4.0 with APG orthogonal projection~\cite{apg} \\
    Sampling steps & 32 \\
    \bottomrule
    \end{tabular}
}
\vspace{2ex}
\end{table}

\textbf{Latent baseline.} For the latent finetuning baseline, we use its native flow shift of 7.0. Other settings are the same as AsymFlow for strict comparability.

\textbf{DDT baseline.} For the DDT pixel finetuning baseline, the DDT head uses two transformer blocks with a wider dimension of 32 attention heads \texttimes 192 features per head, similar to the RAE design~\cite{zheng2025rae}. We use the same $\mA$ matrix as AsymFlow to initialize the input projection layer of the backbone, which closes the input gap and significantly improves the DDT baseline over a random initialization. The DDT head, input/output layers, and LoRA adapters are trained using a common learning rate of $10^{-4}$. Other settings are the same as AsymFlow for strict comparability.

\textbf{Scale calibration ablation.}
\label{app:scale_ablation}
Table~\ref{tab:app_scale_ablation} ablates scale calibration (Appendix~\ref{app:calibration}) using the standard FM loss, under the same 10K-iteration protocol and COCO-10K evaluation as Table~\ref{tab:flux_ablation}. The two runs perform nearly identically, indicating that FLUX.2 klein happens to be insensitive to the input scale. We nevertheless keep scale calibration in the pipeline, since matching the input variance is generally a good practice for transfer learning and comes at negligible cost.

\begin{table}[h]
\centering
\caption{\textbf{Scale calibration ablation.} Both models are finetuned with the standard FM loss on LAION-Aesthetics for 10K iterations and evaluated on COCO-10K.}
\label{tab:app_scale_ablation}
\scalebox{0.8}{%
    \setlength{\tabcolsep}{1.2em}
    \begin{tabular}{lcccccc}
    \toprule
    \textbf{Method}
    & \textbf{HPSv3\textuparrow}
    & \textbf{HPSv2.1\textuparrow}
    & \textbf{VQA\textuparrow}
    & \textbf{CLIP\textuparrow}
    & \textbf{FID\textdownarrow}
    & \textbf{pFID\textdownarrow} \\
    \midrule
    AsymFLUX.2 klein (standard FM) & 12.03 & 0.293 & 0.922 & 0.277 & 20.2 & 25.4 \\
    \quad w/o scale calibration & 11.82 & 0.295 & 0.923 & 0.275 & 20.4 & 25.2 \\
    \bottomrule
    \end{tabular}
}
\vspace{1ex}
\end{table}

\textbf{Inference time.}
AsymFLUX.2 klein uses the same number of tokens as the original FLUX.2 klein, so the per-step running time stays exactly the same as the original latent model. Since VAE is not used, the overall generation speed is marginally faster than the latent model.

\section{Mathematical Derivations}
\label{app:math}

\subsection{AsymFlow Decomposition and Recovery}
\label{app:asymflow_deriv}

We first make explicit the rank-$r$ projector properties used throughout the paper. The columns of $\mA\in\R^{D\times r}$ form an orthonormal basis for the chosen low-rank subspace, so $\mA\T\mA=\mI_r$. This orthonormality makes $\mP=\mA\mA\T$ the orthogonal projector onto that subspace. Applying $\mP$ twice is the same as applying it once, so $\mP^2=\mP$. The complementary projector $\mI-\mP$ removes everything in the low-rank subspace, which gives $(\mI-\mP)\mP=\bm{0}$. Together, these properties mean that any vector can be cleanly separated into a low-rank component and an orthogonal component. The notation is summarized as:
\begin{equation}
    \mA\in\R^{D\times r},
    \qquad
    \mA\T\mA=\mI_r,
    \qquad
    \mP=\mA\mA\T,
    \qquad
    \mP^2=\mP,
    \qquad
    (\mI-\mP)\mP=\bm{0}.
    \label{eq:app_projector_identities}
\end{equation}

We now restate the two targets in this notation. The standard velocity target combines full Gaussian noise with the data term. AsymFlow keeps the same full data term, but applies the projector only to the noise term:
\begin{equation}
    \vu \coloneqq \bm{\epsilon}-\vx_0,
    \qquad
    \vu_\mathrm{A} \coloneqq \mP\bm{\epsilon}-\vx_0.
    \label{eq:app_asymflow_target}
\end{equation}

\textbf{Component decomposition.}
Projecting $\vu_\mathrm{A}$ onto the low-rank subspace gives the true low-rank velocity. This branch of AsymFlow is still a velocity target. It contains low-rank noise minus low-rank data:
\begin{equation}
    \mP\vu_\mathrm{A}
    =
    \mP(\mP\bm{\epsilon}-\vx_0)
    =
    \mP\bm{\epsilon}-\mP\vx_0
    =
    \mP(\bm{\epsilon}-\vx_0)
    =
    \mP\vu .
    \label{eq:app_lowrank_component}
\end{equation}
Projecting $\vu_\mathrm{A}$ onto the orthogonal complement removes the noise term entirely. This branch is no longer a velocity target. It is the orthogonal clean-data component up to a minus sign:
\begin{equation}
    (\mI-\mP)\vu_\mathrm{A}
    =
    (\mI-\mP)(\mP\bm{\epsilon}-\vx_0)
    =
    -(\mI-\mP)\vx_0 .
    \label{eq:app_orthogonal_component}
\end{equation}
Together, Eqs.~(\ref{eq:app_lowrank_component}) and~(\ref{eq:app_orthogonal_component}) show that AsymFlow is velocity-like in $\mathrm{Im}(\mP)$ and $\vx_0$-like in $\mathrm{Im}(\mI-\mP)$.

\textbf{Recovery rule.}
The same decomposition gives an exact route from the asymmetric target back to the standard velocity target. The low-rank branch is already in velocity form, so this component is kept directly:
\begin{equation}
    \mP\vu
    =
    \mP\vu_\mathrm{A}.
    \label{eq:app_lowrank_recovery}
\end{equation}
The orthogonal branch is different. Since Eq.~(\ref{eq:app_orthogonal_component}) says that $(\mI-\mP)\vu_\mathrm{A}$ equals the negative clean-data component, the orthogonal clean data is obtained by changing the sign:
\begin{equation}
    (\mI-\mP)\vx_0
    =
    -(\mI-\mP)\vu_\mathrm{A}.
    \label{eq:app_orthogonal_x0}
\end{equation}
This clean-data component is then converted to velocity using the usual $\vx_0$-to-$\vu$ relation. The orthogonal velocity is obtained by subtracting clean data from the noisy input and dividing by the noise level:
\begin{equation}
    (\mI-\mP)\vu
    =
    (\mI-\mP)\frac{\vx_t-\vx_0}{\sigma_t}
    =
    (\mI-\mP)\frac{\vx_t+\vu_\mathrm{A}}{\sigma_t}.
    \label{eq:app_orthogonal_recovery}
\end{equation}
Combining the direct low-rank velocity branch with the converted orthogonal branch gives the full-rank velocity target:
\begin{equation}
    \vu
    =
    \mP\vu_\mathrm{A}
    +
    (\mI-\mP)\frac{\vx_t+\vu_\mathrm{A}}{\sigma_t}.
    \label{eq:app_full_recovery}
\end{equation}
Thus, the asymmetric target itself contains enough information to reconstruct the standard full-rank velocity target exactly.

\textbf{Endpoint cases.}
The rank controls how much of the target is velocity-like. At rank zero, the projector is zero, so AsymFlow becomes full $\vx_0$-prediction up to sign. At full rank, the projector is the identity, so AsymFlow becomes standard velocity prediction:
\begin{equation}
    r=0 \; \Longrightarrow \; \mP=\mO,\ \vu_\mathrm{A}=-\vx_0,
    \qquad
    r=D \; \Longrightarrow \; \mP=\mI,\ \vu_\mathrm{A}=\bm{\epsilon}-\vx_0=\vu.
    \label{eq:app_endpoints}
\end{equation}

\subsection{Latent--Pixel Flow Coupling at Initialization}
\label{app:latent_pixel_equiv}

We next show the trajectory coupling relationship that makes latent-to-pixel initialization exact: when the latent and lifted pixel ODEs start from paired noise, the entire low-rank pixel trajectory can be lifted from the latent trajectory plus the analytically determined orthogonal noise component. This trajectory coupling holds for both scale-calibrated (Appendix~\ref{app:calibration}) and uncalibrated AsymFlows. Below we analyze the uncalibrated version for simplicity.

Let $\vz_0\in\R^d$ denote a latent token, where $d$ is the latent dimension. In this construction we choose the pixel low-rank subspace to have the same rank $r=d$, and use a linear lift $\mA\in\R^{D\times d}$ from latent tokens to pixel patches. As before, the columns of $\mA$ are orthonormal, so $\mA\T\mA=\mI_d$ and $\mP=\mA\mA\T$ projects onto the latent-induced pixel subspace. The lifted low-rank pixel target is $\vx_0^\mathrm{L} \coloneqq \mA\vz_0$, and projecting pixel noise back through $\mA\T$ gives the latent noise $\bm{\epsilon}_\vz \coloneqq \mA\T\bm{\epsilon}$. The notation is summarized as:
\begin{equation}
    \mA\in\R^{D\times d},
    \qquad
    \mA\T\mA=\mI_d,
    \qquad
    \mP=\mA\mA\T,
    \qquad
    \vx_0^\mathrm{L} \coloneqq \mA\vz_0,
    \qquad
    \bm{\epsilon}_\vz \coloneqq \mA\T\bm{\epsilon}.
    \label{eq:app_latent_pixel_setup}
\end{equation}
With these definitions, projecting the lifted low-rank pixel process recovers the pretrained latent process.

\textbf{Input identity.}
The pixel forward process diffuses the lifted low-rank pixels with full-rank pixel-space noise:
\begin{equation}
    \vx_t^\mathrm{L}
    \coloneqq
    \alpha_t\vx_0^\mathrm{L}+\sigma_t\bm{\epsilon}
    =
    \alpha_t\mA\vz_0+\sigma_t\bm{\epsilon}.
    \label{eq:app_lowrank_forward}
\end{equation}
Projecting this noisy pixel sample by $\mA\T$ gives exactly the corresponding noisy latent sample:
\begin{equation}
    \mA\T\vx_t^\mathrm{L}
    =
    \alpha_t\mA\T\mA\vz_0+\sigma_t\mA\T\bm{\epsilon}
    =
    \alpha_t\vz_0+\sigma_t\bm{\epsilon}_\vz
    =
    \vz_t.
    \label{eq:app_input_identity}
\end{equation}
Thus, the lifted pixel model evaluates the pretrained latent network at the paired noisy latent state.

\textbf{Output identity.}
The latent model predicts latent velocity $\vu_\vz \coloneqq \bm{\epsilon}_\vz-\vz_0$. Lifting this prediction to pixel space gives an AsymFlow-like target for the low-rank pixels $\vx_0^\mathrm{L}$:
\begin{equation}
    \mA\vu_\vz
    =
    \mA(\bm{\epsilon}_\vz-\vz_0)
    =
    \mA\mA\T\bm{\epsilon}-\mA\vz_0
    =
    \mP\bm{\epsilon}-\vx_0^\mathrm{L}.
    \label{eq:app_lifted_latent_velocity}
\end{equation}
Therefore the low-rank pixel velocity $\vu^\mathrm{L} \coloneqq \bm{\epsilon}-\vx_0^\mathrm{L}$ is obtained by applying the same recovery rule from Sec.~\ref{app:asymflow_deriv} with $\vu_\mathrm{A}=\mA\vu_\vz$ and $\vx_t=\vx_t^\mathrm{L}$:
\begin{equation}
    \vu^\mathrm{L} 
    =
    \mP\mA\vu_\vz
    +
    (\mI-\mP)\frac{\vx_t^\mathrm{L}+\mA\vu_\vz}{\sigma_t}.
    \label{eq:app_output_identity}
\end{equation}
For analyzing the lifted latent initialization, this expression can be simplified because the lifted latent prediction already lies in the low-rank subspace, so we have $(\mI-\mP)\mA\vu_\vz=\bm{0}$. This gives
\begin{equation}
    \vu^\mathrm{L}
    =
    \mA\vu_\vz
    +
    \frac{(\mI-\mP)\vx_t^\mathrm{L}}{\sigma_t}.
    \label{eq:app_output_identity_init_simplified}
\end{equation}
Thus, at initialization, the low-rank branch is exactly the lifted latent velocity, while the orthogonal branch is recovered directly from the current noisy pixel state. Note that this simplification does not apply to the finetuned AsymFlow model and should not be used in the implementation.

\textbf{Trajectory coupling.}
The identities above are pointwise statements about the noisy input and the recovered velocity. What we need for initialization is slightly stronger: if the latent model and the lifted pixel model are solved in parallel from paired noise, then their whole trajectories remain paired, and their final samples still satisfy the same lifting relation.
\begin{theorem}
\label{thm:trajectory_consistency}
Let $\bm{\epsilon}\in\R^D$ be a pixel-space noise sample and let $\bm{\epsilon}_\vz=\mA\T\bm{\epsilon}$ be its low-rank projection. Let $G_{\bm{\phi}}$ denote the pretrained latent flow velocity network. Consider the latent flow ODE on $(0,1]$:
\begin{equation}
    \frac{\diff \vz_t}{\diff t}
    =
    G_{\bm{\phi}}(\vz_t,t),
    \qquad
    \vz_1=\bm{\epsilon}_\vz,
    \label{eq:app_latent_ode}
\end{equation}
and the lifted pixel flow ODE obtained by applying the simplified form in Eq.~(\ref{eq:app_output_identity_init_simplified}) to the latent network output:
\begin{equation}
    \frac{\diff \vx_t^\mathrm{L}}{\diff t}
    =
    \mA G_{\bm{\phi}}(\mA\T\vx_t^\mathrm{L},t)
    +
    \frac{(\mI-\mP)\vx_t^\mathrm{L}}{\sigma_t},
    \qquad
    \vx_1^\mathrm{L}=\bm{\epsilon}.
    \label{eq:app_lifted_pixel_ode}
\end{equation}
Then the two trajectories satisfy
\begin{equation}
    \vx_t^\mathrm{L}
    =
    \mA\vz_t+\sigma_t(\mI-\mP)\bm{\epsilon}
    \quad\text{for all }t\in(0,1].
    \label{eq:app_trajectory_identity}
\end{equation}
In particular, taking $t\to0$ gives the final sample identity $\vx_0^\mathrm{L}=\mA\vz_0$.
\end{theorem}
\begin{proof}
For brevity, write the orthogonal noise component as $\bm{\epsilon}^{\perp} \coloneqq (\mI-\mP)\bm{\epsilon}$.
Then the pixel noise decomposes into the lifted latent noise plus the orthogonal residual:
\begin{equation}
    \bm{\epsilon}
    =
    \mP \bm{\epsilon} + (\mI-\mP)\bm{\epsilon}
    =
    \mA\mA\T\bm{\epsilon} + \bm{\epsilon}^{\perp}
    =
    \mA\bm{\epsilon}_\vz + \bm{\epsilon}^{\perp}.
    \label{eq:app_noise_decomposition}
\end{equation}
At $t=1$, this decomposition matches the two ODE initial conditions:
\begin{equation}
    \vx_1^\mathrm{L}
    =
    \mA\vz_1+\sigma_1\bm{\epsilon}^{\perp}.
    \label{eq:app_initial_identity}
\end{equation}
Now define a candidate lifted pixel trajectory from the latent trajectory:
\begin{equation}
    \tilde{\vx}_t^\mathrm{L}
    \coloneqq
    \mA\vz_t+\sigma_t\bm{\epsilon}^{\perp}.
    \label{eq:app_candidate_lifted_path}
\end{equation}
We will show that this candidate trajectory satisfies the lifted pixel ODE in Eq.~(\ref{eq:app_lifted_pixel_ode}) with the same initial condition, so by uniqueness of ODE solutions, it must be identical to $\vx_t^\mathrm{L}$ for all $t$.
The candidate trajectory has exactly the input identity required by the latent network:
\begin{equation}
    \mA\T\tilde{\vx}_t^\mathrm{L}
    =
    \mA\T\mA\vz_t+\sigma_t\mA\T \bm{\epsilon}^{\perp}
    =
    \vz_t.
    \label{eq:app_candidate_projection}
\end{equation}
It also has an orthogonal component determined only by the fixed orthogonal noise:
\begin{equation}
    (\mI-\mP)\tilde{\vx}_t^\mathrm{L}
    =
    \sigma_t\bm{\epsilon}^{\perp}.
    \label{eq:app_candidate_orthogonal}
\end{equation}
Substituting these two identities into the lifted pixel vector field gives the lifted latent velocity plus the orthogonal noise velocity:
\begin{equation}
    \mA G_{\bm{\phi}}(\mA\T\tilde{\vx}_t^\mathrm{L},t)
    +
    \frac{(\mI-\mP)\tilde{\vx}_t^\mathrm{L}}{\sigma_t}
    =
    \mA G_{\bm{\phi}}(\vz_t,t)+\bm{\epsilon}^{\perp}.
    \label{eq:app_lifted_pixel_rhs}
\end{equation}
The derivative of the candidate trajectory gives the same expression:
\begin{equation}
    \frac{\diff \tilde{\vx}_t^\mathrm{L}}{\diff t}
    =
    \mA\frac{\diff \vz_t}{\diff t}
    +
    \frac{\diff\sigma_t}{\diff t}\bm{\epsilon}^{\perp}
    =
    \mA G_{\bm{\phi}}(\vz_t,t)+\bm{\epsilon}^{\perp},
    \label{eq:app_lifted_path_derivative}
\end{equation}
where we used Eq.~(\ref{eq:app_latent_ode}) and $\sigma_t=t$. Thus $\tilde{\vx}_t^\mathrm{L}$ satisfies the lifted pixel ODE in Eq.~(\ref{eq:app_lifted_pixel_ode}). Since it also has the same value as $\vx_t^\mathrm{L}$ at $t=1$, uniqueness of the ODE solution gives
\begin{equation}
    \vx_t^\mathrm{L}
    =
    \tilde{\vx}_t^\mathrm{L}
    =
    \mA\vz_t+\sigma_t(\mI-\mP)\bm{\epsilon}
    \quad\text{for all }t\in(0,1].
    \label{eq:app_trajectory_identity_proof}
\end{equation}
Finally, taking $t\to0$ gives $\vx_0^\mathrm{L} = \mA\vz_0.$
\end{proof}
The same argument applies to Euler discretization with a shared time grid: if the relation holds before a step, the latent update changes the low-rank component by $\Delta t\,\mA G_{\bm{\phi}}(\vz_t,t)$, while the lifted pixel update additionally changes the orthogonal component by $\Delta t\,\bm{\epsilon}^{\perp}$, preserving the same paired form after the step; by induction, the relation holds at all steps. Thus, at network initialization, the lifted latent model is an exact low-rank pixel flow model. Note that this initialization is not yet a full AsymFlow model on real pixels, as finetuning replaces the lifted low-rank data target $\vx_0^\mathrm{L}$ with the full-rank pixel target $\vx_0$.

\subsection{Details on Variance-Reduced Loss}
\label{app:vr_deriv}

The variance-reduced loss in Sec.~\ref{sec:anchored_fm} can be viewed as a control variate. The paired low-rank target $\vx_0^\mathrm{L}$ is correlated with the full pixel target $\vx_0$, and a frozen initialized low-rank model gives a good estimate of it. We use this paired target to reduce the variance of the pixel residual without changing the conditional mean target.

The exact control-variate identity is
\begin{equation}
    \E\!\left[
    \vx_0^\mathrm{L}
    -
    \E[\vx_0^\mathrm{L}| \vx_t]
    \,\middle|\,
    \vx_t
    \right]
    =
    \bm{0}.
    \label{eq:app_control_variate_zero_mean}
\end{equation}
Therefore adding any coefficient times this zero-mean residual does not change the conditional target. The posterior mean remains unchanged, while the sampled target can have lower variance:
\begin{equation}
    \E\!\left[
    \vx_0
    -
    \lambda\bigl(\vx_0^\mathrm{L}-\E[\vx_0^\mathrm{L}| \vx_t]\bigr)
    \,\middle|\,
    \vx_t
    \right]
    =
    \E[\vx_0| \vx_t].
    \label{eq:app_control_variate_target}
\end{equation}
Before approximation, the objective is therefore equivalent to the standard flow matching loss in $\vx_0$ format (Eq.~(\ref{eq:fm_loss})). The only role of the additional term is to reduce sampling variance when the low-rank residual explains part of the full pixel residual.

In practice, the conditional mean $\E[\vx_0^\mathrm{L} | \vx_t]$ is unavailable. We approximate it using the frozen low-rank model prediction $\hat{\vx}_0^\mathrm{L}$ from the paired noisy low-rank sample:
\begin{equation}
    \vx_t^\mathrm{L}
    =
    \alpha_t\vx_0^\mathrm{L}+\sigma_t\bm{\epsilon},
    \qquad
    \E[\vx_0^\mathrm{L}| \vx_t]
    \approx
    \E[\vx_0^\mathrm{L}| \vx_t^\mathrm{L}]
    \approx
    \hat{\vx}_0^\mathrm{L}
    = \mP \vx_t^\mathrm{L} - \sigma_t \mA G_{\bm{\phi}}(\mA\T\vx_t^\mathrm{L},t).
    \label{eq:app_vr_approx}
\end{equation}
Substituting this approximation gives the practical variance-reduced loss in Eq.~(\ref{eq:anchored_loss}).

The approximation $\E[\vx_0^\mathrm{L}| \vx_t] \approx \E[\vx_0^\mathrm{L}| \vx_t^\mathrm{L}]$
is exact under the sufficient condition that the full noisy input and the paired low-rank noisy input differ only in the orthogonal complement. In that case, their low-rank components match, so the frozen low-rank model receives the same low-rank information:
\begin{equation}
    \vx_t-\vx_t^\mathrm{L}\in\mathrm{Im}(\mI-\mP)
    \; \Longrightarrow \;
    \mA\T \vx_t=\mA\T \vx_t^\mathrm{L}.
    \label{eq:app_vr_exact_condition}
\end{equation}
This requires either $t=1$ or $\vx_0 - \vx_0^\mathrm{L}\in\mathrm{Im}(\mI-\mP)$, which is generally not satisfied due to the non-linearity of the VAE encoder~\cite{rombach2022ldm}.
When this condition is not satisfied, the approximation error appears inside the low-rank subspace $\mathrm{Im}(\mP)$. To compensate for this, the perceptual correction is introduced in the low-noise regime in place of the variance reduction, as detailed in Sec.~\ref{app:perceptual}.

\section{Additional Qualitative Results}
\label{app:qualitative}

Fig.~\ref{fig:timestep_viz} visualizes the network outputs of AsymFLUX.2 klein at intermediate steps of the sampler. Fig.~\ref{fig:supp_1} and Fig.~\ref{fig:supp_2} show additional qualitative comparisons of AsymFLUX.2 klein with PixelDiT~\cite{yu2025pixeldit}, FLUX.2 klein~\cite{flux2}, and Qwen Image~\cite{qwen}.

\begin{figure}[H]
\vspace{3ex}
\centering
\includegraphics{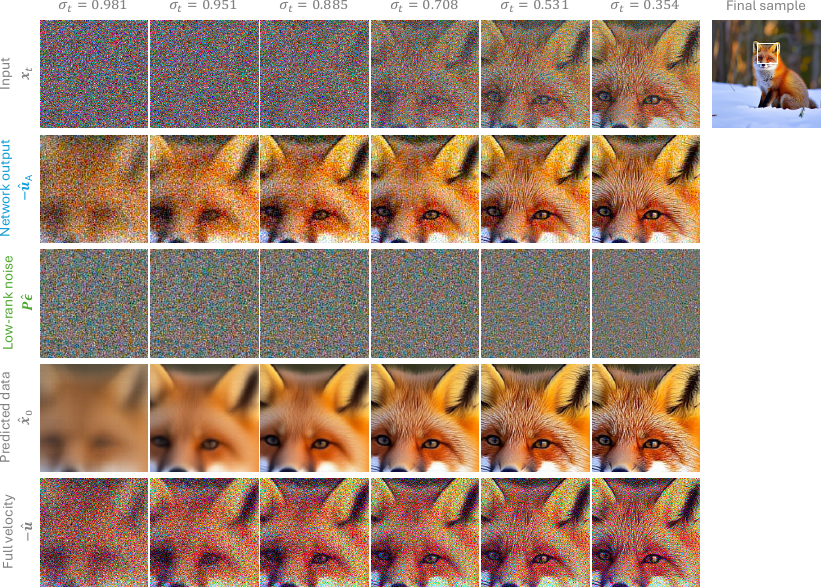}
\caption{\textbf{Network outputs of AsymFLUX.2 klein along the sampling trajectory.} For clarity, we show all quantities in the uncalibrated parameterization of Eq.~(\ref{eq:asymflow_target}) rather than the calibrated one in Eq.~(\ref{eq:app_calibrated_asym_target}). Each column corresponds to a step of the 32-step UniPC sampler at noise level $\sigma_t$, and each row shows a zoomed-in of the indicated quantity. Outputs are taken from the conditional branch before classifier-free guidance. Velocity-related quantities are shown with flipped signs.}
\label{fig:timestep_viz}
\end{figure}

\begin{figure}[p]
\centering
\includegraphics[width=0.92\textwidth]{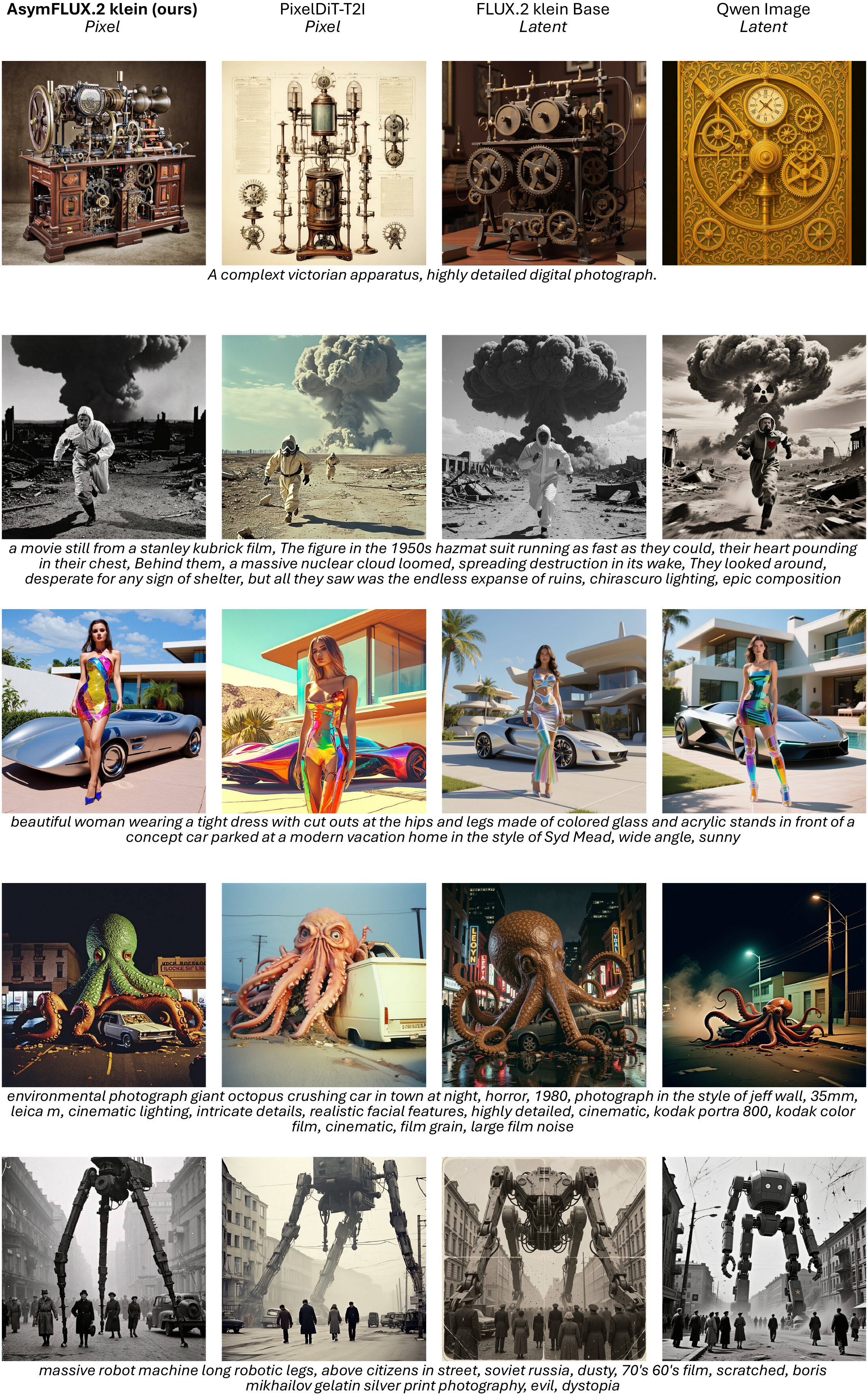}
\caption{Additional qualitative text-to-image comparisons (part A).}
\label{fig:supp_1}
\vspace{2ex}
\end{figure}

\begin{figure}[p]
\centering
\includegraphics[width=0.92\textwidth]{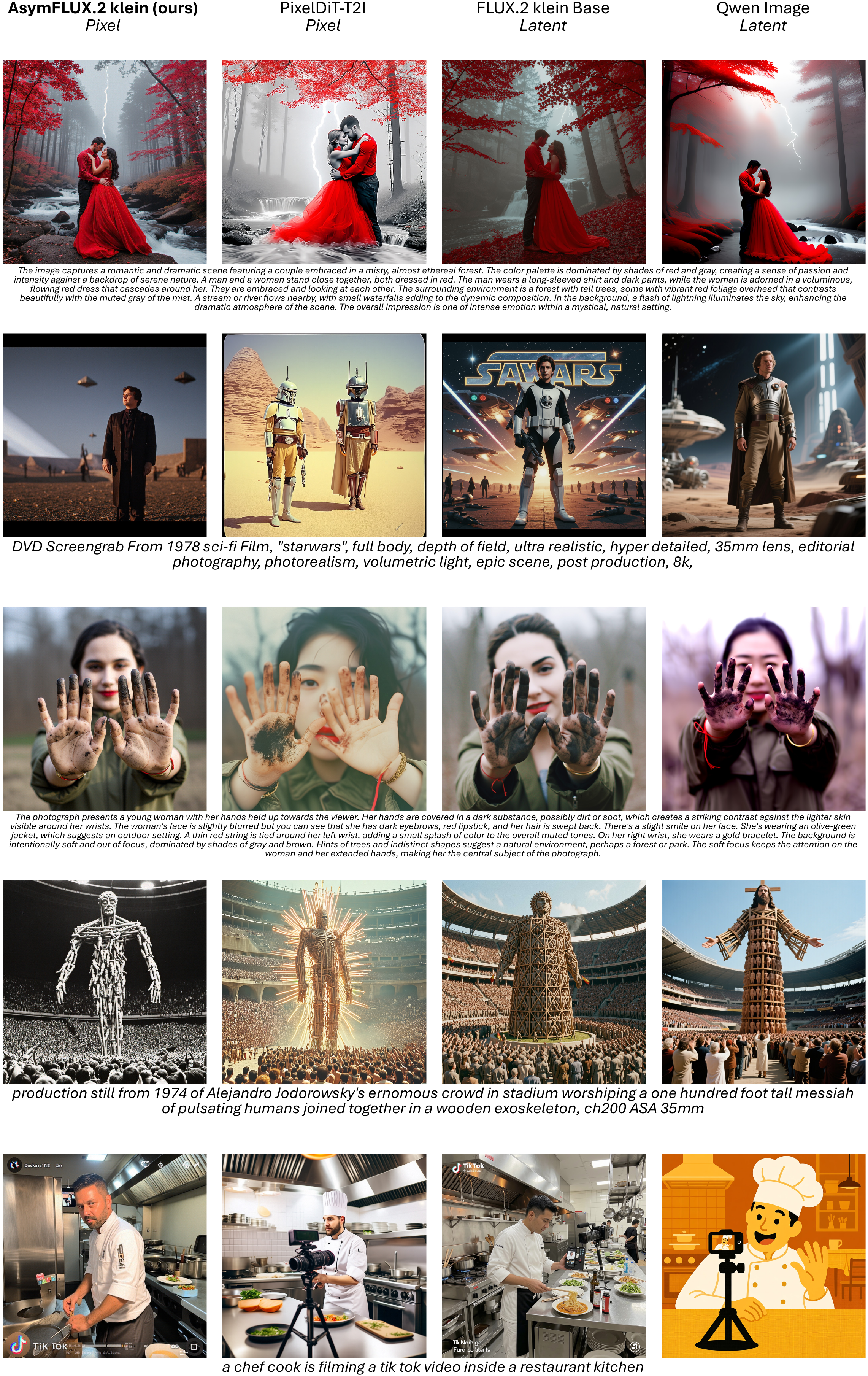}
\vspace{2ex}
\caption{Additional qualitative text-to-image comparisons (part B).}
\label{fig:supp_2}
\end{figure}

\clearpage
\section{Impact Statement}
\label{app:impact}

Our method enhances the photorealism of diffusion models, which significantly benefits creative industries by enabling high-fidelity prototyping and asset creation. This advancement, however, presents a dual-use challenge: more realistic imagery facilitates the creation of convincing disinformation or non-consensual media, increasing the potential for societal harm. Higher visual quality also requires renewed scrutiny of dataset biases, as those biases will be rendered more persuasively. We open-source our model to encourage scientific replication, but emphasize that responsible deployment requires the use of standard safety filters and content provenance tools (like watermarking) to manage these risks.

\end{document}